\RequirePackage{etoolbox}

\documentclass[ba,noinfoline]{imsart}
\pubyear{0000}
\volume{00}
\issue{0}
\doi{0000}
\firstpage{1}
\lastpage{1}

\usepackage{amsthm}
\usepackage{amsmath}
\usepackage{natbib}
\usepackage[breaklinks=true,colorlinks,citecolor=blue,urlcolor=blue,filecolor=blue,backref=page]{hyperref}
\usepackage{graphicx}

\startlocaldefs

\usepackage{amssymb}
\usepackage{caption}
\usepackage{subcaption}

\DeclareMathOperator{\Tr}{tr}

\numberwithin{equation}{section}

\usepackage{enumitem}
\setlist[enumerate]{noitemsep}		

\newtheoremstyle{cited}%
{\topsep}
{\topsep}
{\itshape}
{0pt}
{\bfseries}
{.}
{5pt plus 1pt minus 1pt}
{\thmname{#1}\thmnumber{ #2}\thmnote{ {\normalfont{#3}}}}

\theoremstyle{cited}
\newtheorem{mythm}{Theorem}
\newtheorem{mylem}{Lemma}
\newtheorem{mycor}{Corollary}
\newtheorem{mydef}{Definition}

\endlocaldefs

\begin{document}


\begin{frontmatter}
\title{On the Consistency of Optimal Bayesian Feature Selection in the Presence of Correlations}

\runtitle{On the Consistency of Optimal Bayesian Feature Selection}

\begin{aug}

\author{\fnms{Ali} \snm{Foroughi pour}\thanksref{addr1}\ead[label=e1]{foroughipour.1@osu.edu}}
\and
\author{\fnms{Lori A.} \snm{Dalton}\thanksref{addr1}\ead[label=e2]{dalton@ece.osu.edu}}

\runauthor{A. Foroughi pour and L. A. Dalton}

\address[addr1]{Department	of Electrical and Computer Engineering, The Ohio State University, Columbus, OH, 43210
	\printead{e1}
	\printead{e2}
}

\end{aug}

\begin{abstract}
Optimal Bayesian feature selection (OBFS) is a multivariate supervised screening method designed from the ground up for biomarker discovery. In this work, we prove that Gaussian OBFS is strongly consistent under mild conditions, and provide rates of convergence for key posteriors in the framework.  These results are of enormous importance, since they identify precisely what features are selected by OBFS asymptotically, characterize the relative rates of convergence for posteriors on different types of features, provide conditions that guarantee convergence, justify the use of OBFS when its internal assumptions are invalid, and set the stage for understanding the asymptotic behavior of other algorithms based on the OBFS framework.  
\end{abstract}



\end{frontmatter}

\thispagestyle{empty}


\section{Introduction}

Biomarker discovery aims to identify biological markers (genes or gene products) that lead to improved diagnostic or prognostic tests, better treatment recommendations, or advances in our understanding of the disease or biological condition of interest~\citep{ilyin_biomarker_2004, ref18, ref17}.  
Reliable and reproducible biomarker discovery has proven to be difficult~\citep{diamandis_cancer_2010}; while one reason is that preliminary small-sample datasets are inherently difficult to analyze, another factor is that most popular algorithms and methods employed in bioinformatics have inherent limitations that make them unsuitable for many discovery applications. 

Consider univariate filter methods like t-tests, which are perhaps the most ubiquitous throughout the bioinformatics literature.  Since they focus on only one feature at a time, filter methods cannot take advantage of potential correlations between markers.  In particular, they cannot identify pairs (or sets) of markers with tell-tale behaviors only when observed in tandem.  
Multivariate methods, on the other hand, can account for correlations, but the vast majority are wrapper or embedded feature selection algorithms designed to aid in classification or regression \emph{model reduction}.  Model construction is not a reasonable goal in most small-scale exploratory studies, particularly in biology where the expected number of variables is large and the nature of their interactions can be highly complex and context dependent.  Furthermore, feature selection methods designed for model reduction intrinsically penalize redundant features and reward smaller feature sets~\citep{sima2008peaking, awada2012review, ang2016supervised, li2017feature}.  This can be counterproductive in biomarker discovery; so much so that filter methods often far outperform multivariate methods~\citep{sima2006should,sima2008peaking,foroughipour2018optimal}.  

\emph{Optimal Bayesian feature selection} (OBFS) is a supervised multivariate screening method designed to address these problems~\citep{dalton2013optimal,igib1}.  
The OBFS modeling framework, discussed in detail in Section~\ref{sec:OBFS}, assumes that features can be partitioned into two types of independent blocks: blocks of correlated ``good'' features that have distinct joint distributions between two classes, and blocks of correlated ``bad'' features that have identical joint distributions in all classes~\citep{dalton2018heuristic}.  
Given distributional assumptions on each block, and priors on the distribution parameters, OBFS can be implemented using one of many selection criteria, for example the \emph{maximum number correct} (MNC) criterion, which maximizes the posterior expected number of features correctly identified as belonging to good versus bad blocks, or the \emph{constrained MNC} (CMNC) criterion, which additionally constrains the number of selected features.  
In this way, OBFS aims to optimally identify and rank the set of \emph{all} features with distinct distributions between the classes, which represents the predicted set of biomarkers, while accounting for correlations.  

\emph{Optimal Bayesian feature filtering} (OBF) is a special case of OBFS that assumes that features are independent (blocks of size one) and the parameters governing each feature are independent~\citep{ref8}.  
Gaussian OBF (which assumes independent Gaussian features with conjugate priors, henceforth referred to as simply OBF), has very low computation cost, robust performance when its modeling assumptions are violated, and particularly excels at identifying strong markers with low correlations~\citep{ref10,dalton2018heuristic}.  
Furthermore, it has been shown that OBF is \emph{strongly consistent} under very mild conditions, including cases where the features are dependent and non-Gaussian~\citep{igib1}.  
In particular, OBF, in the long run, selects precisely the set of features with distinct means or variances between the classes.  The consistency theorem in~\citep{igib1} formalizes our intuition that OBF cannot take advantage of correlations; although OBF identifies features that are individually discriminating, we see that it has no capacity to identify features that only discriminate when grouped together, or features that are merely correlated with (or linked through a chain of correlation with) other discriminating features.  

In bioinformatics, features that only discriminate when grouped together and features that are correlated with other strong markers are of tremendous interest, since they may represent important actors in the biological condition under study.  While multivariate methods have the potential to discover such features, as discussed above, methods focused on model reduction tend to do so unreliably and inconsistently.  Rather than involve classification or regression models, OBFS searches for intrinsic differences between features.  Like most multivariate methods, one difficulty with OBFS is that it is computationally expensive (except in certain special cases, like OBF).  If an exact solution is desired, even Gaussian OBFS (which assumes independent \emph{blocks} of Gaussian features with conjugate priors, henceforth referred to as simply OBFS) is currently only tractable in small-scale problems with up to approximately ten features.  This has lead to the development of a number of promising computationally efficient heuristic algorithms based on OBFS theory, for example 2MNC-Robust~\citep{pour2014optimal}, POFAC, REMAIN, and SPM~\citep{dalton2018heuristic}.  

Although OBFS and heuristic algorithms based on OBFS have demonstrated better performance than other multivariate methods in biomarker discovery, it is currently unknown precisely what features they select in the long run.  Our main contribution is a theorem presented in Section~\ref{sec:consistency}, analogous to the consistency theorem presented in~\citep{igib1} for OBF, that shows that OBFS is strongly consistent under mild conditions.  The consistency proof for OBFS utilizes nine lemmas and is significantly more complex than that of OBF.  This theorem identifies precisely what features are selected by OBFS asymptotically, provides conditions that guarantee convergence, justifies the use of OBFS in non-design settings where its internal assumptions are invalid, and characterizes rates of convergence for key posteriors in the framework, including marginal posteriors on different types of features.  

The asymptotic behavior of optimal feature selection provides a frame of reference to understand the performance of heuristic algorithms based on OBFS; for instance, we may compare the sets of features selected asymptotically, the conditions required to guarantee convergence, and rates of convergence.  
Furthermore, while numerous works emphasize the need to account for gene interactions and to detect families of interacting biomarkers [e.g., see~\cite{han2017synergistic}, \cite{xi2018novel} and~\cite{fang2019discovering}], typically the focus is on simple statistics that measure pairwise interactions in the data, rather than on establishing intrinsic characteristics of complete marker families, or on quantifying the performance and properties of selection algorithms designed to identify marker families.  
Thus, this work is also important because it proposes a formal definition for marker families (correlated blocks of features), and shows that these marker families are identifiable (via OBFS).  

\section{Gaussian Optimal Bayesian Feature Selection}
\label{sec:OBFS}

We begin with an overview of the Gaussian modeling framework presented in~\cite{dalton2018heuristic}, and a derivation of the corresponding optimal selection rule, OBFS.  
Let $F$ be the set of feature indices, each corresponding to a real-valued feature.  
We call feature indices we wish to select ``true \emph{good} features,'' denoted by $\bar{G}$, and we call feature indices we wish to not select ``true \emph{bad} features,'' denoted by $\bar{B}=F \backslash \bar{G}$.  
In this model, good features assign (w.p.~$1$ over the prior) distinct distributions between two classes, labeled $y = 0$ and $y = 1$, while bad features assign the same distribution across the whole sample.  
We further assume that $\bar{G}$ and $\bar{B}$ can each be partitioned into sub-blocks of interacting features.  We call the set of all sub-blocks a ``true feature partition,'' denoted by $\bar{P}$.  If $\bar{G}$ and $\bar{B}$ are non-empty, we write $\bar{P} = (\{{\bar{G}_1},\ldots,{\bar{G}_{\bar{u}}}\},\{{\bar{B}_1},\ldots,{\bar{B}_{\bar{v}}}\})$, where $\bar{u}$ and $\bar{v}$ are positive integers, and the set of all $\bar{G}_i$'s and $\bar{B}_j$'s are disjoint such that $\bar{G} = \cup_{i=1}^{\bar{u}} \bar{G}_i$, $\bar{B}=\cup_{j=1}^{\bar{v}} \bar{B}_j$, and all $\bar{G}_i$'s and $\bar{B}_j$'s are non-empty.  
If $\bar{G}$ is empty, then we still denote $\bar{P}$ this way, but also define $\cup G_i = \varnothing$, and define sums and products from $1$ to $\bar{u}$ to be $0$ and $1$, respectively.  We define similar conventions when $\bar{B}$ is empty.  

In the Bayesian model, $\bar{G}$, $\bar{B}$ and $\bar{P}$ are random sets, and we denote instantiations of these random sets by $G$, $B$, and $P = (\{{G_1},\ldots,{G_u}\},\{{B_1},\ldots,{B_v}\})$, respectively.  We define a prior on the true feature partition, $\pi(P) \equiv \text{P}(\bar{P} = P)$, which induces a prior on the true good features, $\pi(G) \equiv \text{P}(\bar{G} = G) = \sum_{P:\cup G_i = G} \pi(P)$.  

Given a true feature partition $\bar{P} = P$, define $\theta^{G_i}_y$ for each class $y = 0, 1$ and each good block index $i = 1, \ldots, u$, and $\theta^{B_j}$ for each bad block index $j = 1, \ldots v$.  Let $\theta^{P}$ denote the collection of all $\theta^{G_i}_y$'s and $\theta^{B_j}$'s.  Assume the $\theta^{G_i}_y$'s and $\theta^{B_j}$'s are mutually independent, i.e., we have a prior of the form
\begin{equation}
\pi(\theta^P|\bar{P}=P)=\prod_{i=1}^u \pi(\theta^{G_i}_0) \pi(\theta^{G_i}_1) \prod_{j=1}^v \pi(\theta^{B_j}).
\end{equation}
For good block $A = G_i$, assume $\theta^{A}_y=[\mu^{A}_y, \Sigma^{A}_y]$ for each class $y = 0, 1$, where $\mu^{A}_y$ is a length $|A|$ column vector, $\Sigma^{A}_y$ is an $|A|\times |A|$ matrix, and $|A|$ denotes the cardinality of set $A$.  
We assume $\pi(\theta^{A}_y)=\pi(\mu^{A}_y|\Sigma^{A}_y)\pi(\Sigma^{A}_y)$ is normal-inverse-Wishart:
\begin{align}
\pi(\Sigma^{A}_y) &= K^{A}_y |\Sigma^{A}_y|^{-0.5(\kappa^{A}_y+|A|+1)} 
\exp\left( -0.5 \Tr( S^{A}_y (\Sigma^{A}_y)^{-1} ) \right), \\
\pi(\mu^{A}_y|\Sigma^{A}_y) &= L^{A}_y  |\Sigma^{A}_y|^{-0.5} 
\exp\left( -0.5 \nu^{A}_y (\mu^{A}_y-m^{A}_y)^T (\Sigma^{A}_y)^{-1} (\mu^{A}_y-m^{A}_y)  \right),
\end{align}
where $|X|$ is the determinant, $\Tr(X)$ is the trace, and $X^T$ is the transpose of matrix $X$.  
For a proper prior, 
$\kappa^{A}_y>|A|-1$, 
$S^{A}_y$ is a symmetric positive-definite $|A| \times |A|$ matrix, 
$\nu^{A}_y>0$, 
$m^{A}_y$ is an $|{A}| \times 1$ vector,
\begin{align}
K^{A}_y&=|S^{A}_y|^{0.5 \kappa^{A}_y}2^{-0.5 \kappa^{A}_y |{A}| } /  \Gamma_{|{A}|}(0.5 \kappa^{A}_y), 
\label{eq:K_y} \\
L^{A}_y&=(2\pi/\nu^{A}_y)^{-0.5|{A}|}, 
\label{eq:L_y}
\end{align}
and $\Gamma_{d}$ denotes the multivariate gamma function, where $d$ is a positive integer. 
Likewise, for bad block $A = B_j$ assume $\theta^{A}=[\mu^{A}, \Sigma^{A}]$, and that  $\pi(\theta^{A})$ is normal-inverse-Wishart with hyperparameters $\kappa^{A}$, $S^{A}$, $\nu^{A}$, and $m^{A}$, and scaling constants $K^{A}$ and $L^{A}$. 

Let $S_n$ be a sample composed of $n$ labeled points with $n_y$ points in class $y$, where labels are determined by a process independent from $\bar{P}$ and $\theta^{\bar{P}}$ (for instance, using \emph{random sampling} or \emph{separate sampling}).  Given $\bar{P} = P$, $\theta^{P}$, and the labels, assume all sample points are mutually independent, and features in different blocks are also independent.  Assume that features in block $G_i$, $i=1, \ldots, u$, are jointly Gaussian with mean $\mu^{G_i}_y$ and covariance matrix $\Sigma^{G_i}_y$ under class $y$, and that features in block $B_j$, $j = 1, \ldots, v$, are jointly Gaussian with mean $\mu^{B_j}$ and covariance $\Sigma^{B_j}$ across the whole sample.  

The posterior on $\bar{P}$ over the set of all valid feature partitions is given by the normalized product of the prior and likelihood function.  It can be shown that $\pi^*(P) \equiv \text{P}(\bar{P}=P|S_n) \propto \pi(P) q(P) a(P)$, where $\pi^*(P)$ is normalized to have unit sum, and
\begin{align}
q(P) &= 
\prod_{\substack{i=1, \ldots, u, \\ y = 0,1}} Q^{G_i}_y (\kappa^{G_i*}_y-|G_i|-1)^{-0.5 |G_i| \kappa^{G_i*}_y} 
\nonumber \\ 
&\quad \times
\prod_{j=1, \ldots, v}
Q^{B_j} (\kappa^{B_j*}-|B_j|-1)^{-0.5 |B_j| \kappa^{B_j*}}, \\
a(P)&= \Bigg(  \prod_{\substack{i=1, \ldots, u, \\ y = 0,1}} |C^{G_i}_y|^{\kappa^{G_i*}_y}
\prod_{j=1, \ldots, v} |C^{B_j}|^{\kappa^{B_j*}} \Bigg) ^{-0.5}.
\end{align}
For each good block $A = G_i$, $i=1, \ldots, u$, and class $y=0, 1$, 
\begin{align}
\nu^{A*}_y&=\nu^{A}_y+n_y, 
\label{eq:nu_y} \\
\kappa^{A*}_y&=\kappa^{A}_y+n_y, 
\label{eq:kappa_y} \\
Q^{A}_y &=K^{A}_y L^{A}_y 2^{0.5 \kappa^{A*}_y |A|} \Gamma_{|A|}(0.5\kappa^{A*}_y) \left( {2 \pi}/{\nu^{A*}_y} \right)^{0.5|A|},
\label{eq:Q_y} \\ 
S^{A*}_y&=S^{A}_y+(n_y-1) \hat{\Sigma}^{A}_y+\frac{\nu^{A}_y n_y}{\nu^{A*}_y} (\hat{\mu}^{A}_y-m^{A}_y)(\hat{\mu}^{A}_y-m^{A}_y)^T, 
\label{eq:S_y} \\
C^A_y&={S^{A*}_y}/{(\kappa^{A*}_y-|A|-1)}, 
\label{eq:C_y}
\end{align}
and $\hat{\mu}^{A}_y$ and $\hat{\Sigma}^{A}_y$ are the sample mean and unbiased sample covariance of features in $A$ under class $y$.  
Similarly, for each bad block $A = B_j$, we find $\nu^{A*}$, $\kappa^{A*}$, $Q^{A}$, $S^{A*}$ and $C^{A}$ using~\eqref{eq:nu_y} through~\eqref{eq:C_y} with all subscript $y$'s removed, where $\hat{\mu}^{A}$ and $\hat{\Sigma}^{A}$ are the sample mean and unbiased sample covariance of features in $A$ across the whole sample.  

The posterior on $\bar{P}$ induces a posterior on $\bar{G}$ over all subsets $G \subseteq F$, 
\begin{align}
\label{eq:marginal0}
\pi^*(G) 
\equiv \text{P}(\bar{G}=G|S_n) 
&=\sum_{P:\cup G_i = G} \pi^*(P),
\end{align}
as well as posterior probabilities that each individual feature $f \in F$ is in $\bar{G}$, 
\begin{align}
\pi^*(f) 
\equiv \text{P}(f \in \bar{G}|S_n) 
&=\sum_{G: f \in G} \pi^*(G).
\label{eq:posterior_f}
\end{align}
Note $\pi^*(f)=\text{P}(f \in \bar{G}|S_n)$ and $\pi^*(\{f\})=\text{P}( \bar{G}=\{f\}|S_n)$ are distinct.  

The objective of OBFS is to identify the set of true good features, $\bar{G}$. We will consider two objective criteria: MNC and CMNC.  MNC maximizes the expected number of correctly identified features; that is, MNC outputs the set $G \subseteq F$ with complement $B = F \backslash G$ that maximizes $\text{E}[|G \cap \bar{G}| + |B \cap \bar{B}|]$, which is given by $G^{\textrm{MNC}} = \{f \in F: \pi^*(f) >0.5 \}$~\citep{pour2014optimal}.  CMNC maximizes the expected number of correctly identified features under the constraint of selecting a specified number of features, $D$, and $G^{\textrm{CMNC}}$ is found by picking the $D$ features with largest $\pi^*(f)$~\citep{fdin}.
Both MNC and CMNC require computing $\pi^*(f)$ for all $f \in F$, which generally requires computing $\pi(P)$ for all valid feature partitions $P$, and is generally intractable unless $|F|$ is small.  

Under proper priors, Gaussian OBFS takes the following modeling parameters as input: (1) $\pi(P)$ for each valid feature partition, $P$, (2) $\nu^{A}_y>0$, $m^A_y$, $\kappa^{A}_y >|A|-1$, and symmetric positive-definite $S^{A}_y$ for all $y$ and all possible good blocks $A$ in valid $P$, and (3) $\nu^{A}>0$, $m^A$, $\kappa^{A} >|A|-1$, and symmetric positive-definite $S^{A}$ for all possible bad blocks $A$ in valid $P$.  If CMNC is used, it also takes $D$ as input.
When $\pi(\theta^P|\bar{P}=P)$ is improper, the above derivations are invalid, but $\pi^*(P) \propto \pi(P) q(P) a(P)$ can still be taken as a definition and computed as long as: (1) $\pi(P)$ is proper, (2) $\nu^{A*}_y>0$, $\kappa^{A*}_y >|A|-1$, and $S^{A*}_y$ is symmetric positive-definite, (3) $\nu^{A*}>0$, $\kappa^{A*}>|A|-1$, and $S^{A*}$ is symmetric positive-definite, and (4) $K^{A}_y$ and $L^{A}_y$ are no longer given by~\eqref{eq:K_y} and~\eqref{eq:L_y}, and similarly $K^{A}$ and $L^{A}$ are no longer given by analogous equations; instead these are positive input parameters specified by the user.  

Gaussian OBF is a special case where $\pi(P)$ assumes that all blocks $\bar{G}_1, \ldots, \bar{G}_{\bar{u}}$ and $\bar{B}_1, \ldots, \bar{B}_{\bar{v}}$ are of size one, and the events $\{f \in \bar{G}\}$ are mutually independent.  
For each $f \in F$, OBF takes as input (1) marginal priors $\pi(f) \equiv \text{P}(f \in \bar{G})$, (2) scalars $\nu_y^{\{f\}}$, $m_y^{\{f\}}$, $\kappa_y^{\{f\}}$ and $S_y^{\{f\}}$ for all $y$, (3) scalars $\nu^{\{f\}}$, $m^{\{f\}}$, $\kappa^{\{f\}}$ and $S^{\{f\}}$, and (4) $L^f \equiv K_0^{\{f\}} L_0^{\{f\}} K_1^{\{f\}} L_1^{\{f\}} / ( K^{\{f\}} L^{\{f\}} )$ if improper priors are used.  OBF also takes $D$ as input if CMNC is used.  
Under OBF, it can be shown that $\pi^*(f)$, defined in~\eqref{eq:posterior_f}, simplifies to $\pi^*(f) = h(f)/(1+h(f))$, where 
\begin{equation}
h(f) =  
\frac{\pi(f)}{1-\pi(f)} 
\cdot \frac{Q_0^{\{f\}}Q_0^{\{f\}}}{Q^{\{f\}}}
\cdot \frac{(S^{\{f\}*})^{0.5\kappa^{\{f\}*}}}
{(S_0^{\{f\}*})^{0.5\kappa_y^{\{f\}*}}(S_1^{\{f\}*})^{0.5\kappa_y^{\{f\}*}}},
\label{eq:OBF_posterior}
\end{equation}
$\kappa_y^{\{f\}*}$, $Q_y^{\{f\}}$ and $S_y^{\{f\}*}$ are defined in~\eqref{eq:kappa_y}, \eqref{eq:Q_y} and~\eqref{eq:S_y}, respectively, and $\kappa^{\{f\}*}$, $Q^{\{f\}}$ and $S^{\{f\}*}$ are defined similarly.  Rather than evaluate $\pi^*(P)$ for all feature partitions, OBF under both MNC and CMNC reduces to simple filtering, where features are scored by the $h(f)$ given in~\eqref{eq:OBF_posterior}.  

\section{Consistency}
\label{sec:consistency}

Let $\mathcal{F}_{\infty}$ be an arbitrary sampling distribution on an infinite sample, $S_{\infty}$.  Each sample point in $S_{\infty}$ consists of a binary label, $y=0,1$, and a set of features corresponding to a set of feature indices, $F$.  
For each $n = 1, 2, \ldots$, let $S_n$ denote a sample consisting of the first $n$ points in $S_{\infty}$, let $n_y$ denote the number of points in class $y$, and define $\rho_n=n_0/n$.  

The goal of feature selection is to identify a specific subset of features (say those with different distributions between two classes), which we denote by $\bar{G}$.  Proving strong consistency for a feature selection algorithm thus amounts to showing that
\begin{align}
\lim_{n \to \infty} \hat{G}(S_n)=\bar{G}
\label{eq:strong_consistency}
\end{align}
with probability~$1$ (w.p.~$1$) over the infinite sampling distribution, where $n$ is the sample size, $S_n$ is a sample of size $n$, and $\hat{G}(S_n)$ is the output of the selection algorithm.
Here, $\hat{G}(S_n)$ and $\bar{G}$ are sets, and we define $\lim_{n \to \infty} \hat{G}(S_n)=\bar{G}$ to mean that $\hat{G}(S_n)=\bar{G}$ for all but a finite number of $n$.  

OBFS and OBF fix the sample and model the set of features we wish to select and the sampling distribution as random. 
To study consistency, we reverse this: now the sampling distribution is fixed and the sample is random.  
We begin by reviewing a few definitions and a special case of the strong consistency theorem for OBF.  

\begin{mydef}[\citep{igib1}]
	\label{def:indep_unambiguous}
	$\bar{G}$ is an \emph{independent unambiguous set of good features} if the following hold, where $\mu^{f}_y$ and $\sigma^{f}_y$ are the mean and variance of feature $f$ under class $y$, respectively:
	\begin{enumerate}
		\item{For each $g \in \bar{G}$, $\mu^{g}_y$ and $\sigma^{g}_y$ exist for all $y$ such that either $\mu^{g}_0 \neq \mu^{g}_1$ or $\sigma^{g}_0 \neq \sigma^{g}_1$.}
		\item{For each $b \notin \bar{G}$, $\mu^{b}_y$ and $\sigma^{b}_y$ exist for all $y$ such that $\mu^{b}_0 = \mu^{b}_1$ and $\sigma^{b}_0 = \sigma^{b}_1$.}
	\end{enumerate}
\end{mydef}

\begin{mydef}[\citep{igib1}]
	\label{def:balanced}
	An infinite sample, $S_{\infty}$, is a \emph{balanced sample} if $\liminf_{n \to \infty} \rho_n >0$, $\limsup_{n \to \infty} \rho_n <1$, and, conditioned on the labels, sample points are independent with points belonging to the same class identically distributed.  
\end{mydef}

\begin{mythm}[\citep{igib1}]
	Let $S_{\infty}$ be a fixed infinite sample and let $G$ be a fixed subset of $F$.  If $\lim_{n \to \infty} \pi^*(G) = 1$, then $\lim_{n \to \infty} \hat{G}_{\mathrm{MNC}}(S_n) = G$  and $\lim_{n \to \infty} \hat{G}_{\mathrm{CMNC}}(S_n) = G$ if $D = |G|$, where $\hat{G}_{\mathrm{MNC}}(S_n)$ and $\hat{G}_{\mathrm{CMNC}}(S_n)$ are the MNC and CMNC feature sets under $\pi^*(G)$, respectively.  
	\label{theorem:main}
\end{mythm}

\begin{mythm}[\citep{igib1}]
Suppose the following are true:

\begin{enumerate}
	\item[(i)] $\bar{G}$ is the independent unambiguous set of good features.  
	\item[(ii)] For each feature not in $\bar{G}$, the fourth order moment across the whole sample exists.  
	\item[(iii)] $S_{\infty}$ is a balanced sample w.p.~$1$.  
	\item[(iv)] $\pi^*(G)$ assumes a fixed Gaussian OBF model for all $n$ with $0 < \pi(f) < 1$ for all $f$.  
\end{enumerate}

\noindent
Then $\lim_{n \to \infty} \pi^*(\bar{G}) = 1$ for $\mathcal{F}_{\infty}$-almost all infinite samples.
\label{theorem:OBF}
\end{mythm}

By Theorems~\ref{theorem:main} and~\ref{theorem:OBF}, Gaussian OBF under MNC and Gaussian OBF under CMNC with ``correct'' $D$ (the user knows in advance how many features to select) are strongly consistent 
if the conditions of Theorem~\ref{theorem:OBF} hold.  
Condition (i) uses Definition~\ref{def:indep_unambiguous}, and characterizes the features that OBF aims to select.  
Condition (iii) uses Definition~\ref{def:balanced}, and characterizes requirements of the sample.  
Conditions (i)--(iii) impose very mild assumptions on the data generating process, which is not required to satisfy the Gaussian OBF modeling assumptions in Condition (iv).  

Let us now turn to OBFS.  
Denote the true mean and covariance of features in feature set $A$ and class $y$ by $\mu^{A}_y$ and $\Sigma^{A}_y$, respectively (the features need not be Gaussian).
OBFS aims to select the smallest set of features, $\bar{G}$, with different means or covariances between the classes, where $\bar{B} = F \backslash \bar{G}$ has the same means and covariances between the classes and is uncorrelated with $\bar{G}$.  
This is formalized in the following definition.  

\begin{mydef}
	\label{sec:ded}
	$\bar{G}$ is an \emph{unambiguous set of good features} if the following hold:
	\begin{enumerate}
		\item $\mu^{F}_y$ and $\Sigma^{F}_y$ exist for all $y$.  
		\item{Either $\mu^{\bar{G}}_0 \neq \mu^{\bar{G}}_1$ or $\Sigma^{\bar{G}}_0 \neq \Sigma^{\bar{G}}_1$.}
		\item{$\mu^{\bar{B}}_0 = \mu^{\bar{B}}_1$ and $\Sigma^{\bar{B}}_0 = \Sigma^{\bar{B}}_1$, where $\bar{B} = F \backslash \bar{G}$.}
		\item{Each feature in $\bar{G}$ is uncorrelated with each feature in $\bar{B}$ in both classes.}
		\item{Conditions 1-4 do not all hold for any strict subset $G \subset \bar{G}$.}
	\end{enumerate}
\end{mydef}

Assuming all first and second order moments exist, an unambiguous set of good features always exists and is unique.  To prove uniqueness, let $\bar{G}_1$ and $\bar{G}_2$ be arbitrary unambiguous sets of good features, and 
let $\bar{G}_3 = \bar{G}_1 \cap \bar{G}_2$.
By condition 4, $\Sigma_y^F$, has a block diagonal structure for each $y=0,1$ corresponding to the sets of features $\bar{G}_3$, $\bar{G}_1 \backslash \bar{G}_3$, $\bar{G}_2 \backslash \bar{G}_3$, and $F \backslash (\bar{G}_1 \cup \bar{G}_2)$.  Thus, condition 4 holds for $\bar{G}_3$.  
By condition 3, $\mu^{A}_0 = \mu^{A}_1$ and $\Sigma^{A}_0 = \Sigma^{A}_1$ for all of these blocks except $\bar{G}_3$.  Thus, condition 3 holds for $\bar{G}_3$.  
If $\mu^{\bar{G}_1}_0 \neq \mu^{\bar{G}_1}_1$, then (since the means for each class are equal for $\bar{G}_1 \backslash \bar{G}_3$) $\mu^{\bar{G}_3}_0 \neq \mu^{\bar{G}_3}_1$.  Alternatively, if $\Sigma^{\bar{G}_1}_0 \neq \Sigma^{\bar{G}_1}_1$, then (since $\bar{G}_3$ and $\bar{G}_1 \backslash \bar{G}_3$ are uncorrelated for each class, and the covariances between each class are equal for $\bar{G}_1 \backslash \bar{G}_3$) $\Sigma^{\bar{G}_3}_0 \neq \Sigma^{\bar{G}_3}_1$.  In either case, condition 2 holds for $\bar{G}_3$.  
Since $\bar{G}_3 \subseteq \bar{G}_1$ and $\bar{G}_3 \subseteq \bar{G}_2$, by condition 5 we must have $\bar{G}_1 = \bar{G}_2 = \bar{G}_3$.
We denote the unique unambiguous set of good features by $\bar{G}$, and its complement by $\bar{B}$, throughout.  

Similar to the Bayesian model, define a feature partition to be an ordered set of the form $P=(\{G_1, \ldots, G_{u}\}, \{B_1, \ldots, B_{v}\})$, where the set of $G_i$'s and $B_j$'s partition $F$.
We call each $G_i$ a ``good block'' and each $B_j$ a ``bad block'', keeping in mind that these terms are always relative to a specified arbitrary partition. 
Feature partitions with no good blocks or no bad blocks are permitted, with appropriate conventions for unions, sums, and products over the (empty set of) corresponding blocks.  
The following definitions formalize a non-Bayesian analog of the true feature partition.  

\begin{mydef}
	Let $P_1$ and $P_2$ be arbitrary feature partitions.  
	$P_1$ is a \emph{mesh} of $P_2$ if every block in $P_1$ is a subset of a block in $P_2$.
	$P_1$ is a \emph{refinement} of $P_2$ if every good block in $P_1$ is a subset of a good block in $P_2$ and every bad block in $P_1$ is a subset of a bad block in $P_2$.  
	$P_1$ is a \emph{strict refinement} of $P_2$ if it is a refinement and $P_1 \neq P_2$.
\end{mydef}

\begin{mydef}
	$\bar{P}=(\{\bar{G}_1, \ldots, \bar{G}_{\bar{u}}\}, \{\bar{B}_1, \ldots, \bar{B}_{\bar{v}}\})$ is an \emph{unambiguous feature partition} if the following hold:
	\begin{enumerate}
		\item $\bar{G} = \cup_{i=1}^{\bar{u}} \bar{G}_i$ is an unambiguous set of good features.  
		\item Each feature in any arbitrary block is uncorrelated with each feature in any other block in both classes.
		\item Conditions 1 and 2 do not hold for any feature partition $P$ that is a strict refinement of $\bar{P}$.
	\end{enumerate}
\end{mydef}

An unambiguous feature partition always exists and is unique, and we denote it by $\bar{P}$ throughout.  By definition, the unambiguous feature partition $\bar{P}$ induces the unambiguous set of good features $\bar{G}$.  

Our main result is given in the following theorem (Theorem~\ref{sec:thm_conv_1}), which provides sufficient conditions for the (almost sure) convergence of $\pi^*(P)$ to a point mass at $\bar{P}$, thereby guaranteeing the (almost sure) convergence of $\pi^*(G)$ to a point mass at $\bar{G}$.  By Theorem~\ref{theorem:main}, Gaussian OBFS under MNC and Gaussian OBFS under CMNC with ``correct'' $D$ are strongly consistent if the conditions of Theorem~\ref{sec:thm_conv_1} hold, which are very mild.  
Condition (i) is based on Definition~\ref{sec:ded} and essentially guarantees that $\bar{G}$ really represents the type of features OBFS aims to select (those with different means or covariances between the classes).  
Conditions (i) and (ii) ensure certain moments exist, and Condition (iii) ensures that all covariances are full rank.  There are no other distributional requirements. 
Condition (iv) is based on Definition~\ref{def:balanced} and characterizes the sampling strategy; the proportion of points observed in any class must not converge to zero, and sample points should be independent.  
Condition (v) places constraints on the inputs to the OBFS algorithm; we must assign a non-zero probability to the true feature partition.  

Finally, from the proof of Theorem~\ref{sec:thm_conv_1} we have that, under the conditions stated in the theorem, w.p.~$1$ there exist $0<r<1$ and $N>0$ such that
\begin{align}
\frac{\pi^*(P)}{\pi^*(\bar{P})} < r^n
\label{eq:rate2}
\end{align}
for all $n > N$ and all $P \neq \bar{P}$ that label any good features as bad or that put features that are correlated in separate blocks.  Also, w.p.~$1$ there exist $s, N>0$ such that
\begin{align}
\frac{\pi^*(P)}{\pi^*(\bar{P})} < n^{-s}
\label{eq:rate1}
\end{align}
for all $n > N$ and all $P \neq \bar{P}$.  
Fix $f \in \bar{G}$.  By~\eqref{eq:rate2}, for each feature partition $P$ such that $f \notin \cup G_i$, w.p.~$1$ there exist $0 < r_P < 1$ and $N_P > 0$ such that $\pi^*(P)/\pi^*(\bar{P}) < r_P^n$ for all $n > N_P$.  Therefore, w.p.~$1$, 
\begin{equation}
\frac{1 - \pi^*(f)}{\pi^*(\bar{P})}
= \sum_{P : f \notin \cup G_i} \frac{\pi^*(P)}{\pi^*(\bar{P})}
< r^n
\end{equation}
for all $n>N$, where $\max_{P : f \notin \cup G_i} r_P < r < 1$ and $N > \max_{P : f \notin \cup G_i} N_P$.  Thus, w.p.~$1$, 
\begin{equation}
\pi^*(f)
> 1 - \pi^*(\bar{P}) r^n
\geq 1 - r^n
\end{equation}
for all $n>N$. The marginal posterior of good features converges to $1$ at least exponentially w.p.~$1$.
Now fix $f \in \bar{B}$.  
By~\eqref{eq:rate1}, w.p.~$1$ there exist $s, N > 0$ such that
\begin{equation}
\frac{\pi^*(f)}{\pi^*(\bar{P})}
= \sum_{P : f \in \cup G_i} \frac{\pi^*(P)}{\pi^*(\bar{P})}
< n^{-s}
\end{equation}
for all $n>N$.  Thus, w.p.~$1$, 
\begin{equation}
\pi^*(f)
< n^{-s}
\end{equation}
for all $n>N$.  In other words, the marginal posterior of bad features converges to $0$ at least polynomially w.p.~$1$.  
This characterizes rates of convergence of the posterior on feature partitions, and the marginal posterior probabilities on individual features.  

Throughout, ``$f$ and $g$ are asymptotically equivalent'' and ``$f \sim g$ as $n \to \infty$'' mean that $\lim_{n \to \infty} f(n)/g(n) = 1$, $v(i)$ denotes the $i$th element of vector $v$, $M(i,j)$ denotes the $i$th row, $j$th column element of matrix $M$, $0_{a,b}$ denotes the all-zero $a \times b$ matrix, and the sample mean and unbiased sample covariance of features in block $A$ and class $y$ are denoted by $\hat{\mu}_y^A$ and $\hat{\Sigma}_y^A$, respectively.  

\begin{mythm}
	\label{sec:thm_conv_1}
	Suppose the following hold:
	\begin{enumerate}
		\item[(i)]{$\bar{P}$ is the unambiguous feature partition.}
		\item[(ii)]{For all $f \in F$, the fourth order moment across both classes exists and is finite.}
		\item[(iii)]{$\Sigma_0^F$ and $\Sigma_1^F$ are full rank.}
		\item[(iv)]{$S_{\infty}$ is a balanced sample w.p.~$1$.}
		\item[(v)]{$\pi^*(P)$ assumes a fixed Gaussian OBFS model for all $n$ with $\pi(\bar{P}) > 0$.}
	\end{enumerate}
	Then $\lim_{n \to \infty} \pi^*(\bar{P}) = 1$ for $\mathcal{F}_{\infty}$-almost all sequences.  
\end{mythm}

\begin{proof}
	Since $\sum_P \pi^*(P) = 1$, it suffices to show that for all feature partitions $P \neq \bar{P}$, 
	\begin{equation}
	\lim_{n\to\infty} \frac{\pi^*(P)}{\pi^*(\bar{P})} 
	= \lim_{n\to\infty} \frac{\pi(P)}{\pi(\bar{P})} \times \frac{q(P)}{q(\bar{P})}	\times \frac{a(P)}{a(\bar{P})}
	= 0 \quad \mbox{w.p.~$1$}.
	\label{eq:hokm_thm_big}
	\end{equation}
	Fix $P \neq \bar{P}$. 
	If $\pi(P)=0$ then~\eqref{eq:hokm_thm_big} holds trivially.  In the remainder of this proof, assume $\pi(P) \neq 0$.  It suffices to show:
	\begin{equation}
	\lim_{n\to\infty} 
	\frac{q(P)}{q(\bar{P})}	\times 
	\frac{a(P)}{a(\bar{P})}
	= 0 \quad \mbox{w.p.~$1$}.
	\label{eq:hokm_thm}
	\end{equation}
	We prove this by constructing a sequence of partitions from $P$ to $\bar{P}$ in six steps.  In Step~1, blocks that are labeled bad in $P$ are split into subsets of bad blocks in $\bar{P}$ and subsets of the unambiguous set of good features $\bar{G}$.  In Steps~2 and~3, blocks that are labeled bad in the previous partition but that are actually subsets of $\bar{G}$ are labeled good (and in Step 3 they are also merged with other good blocks).  In Step~4, blocks that are labeled good in the previous partition are split into subsets of (good and bad) blocks in $\bar{P}$.  In Step~5, blocks that are labeled good in the previous partition but that are actually subsets of bad blocks in $\bar{P}$ are labeled bad.  In Step~6, we merge blocks in the previous partition until we arrive at $\bar{P}$.  An example of this sequence of partitions is illustrated in Fig.~\ref{fig:consistency_optimal_example2}.  Finally, Step~7 uses the sequence of partitions constructed in Steps~1-6 to show~\eqref{eq:hokm_thm}.  
	
	\begin{figure}[t!]
		\centering
		\includegraphics[width=\textwidth]{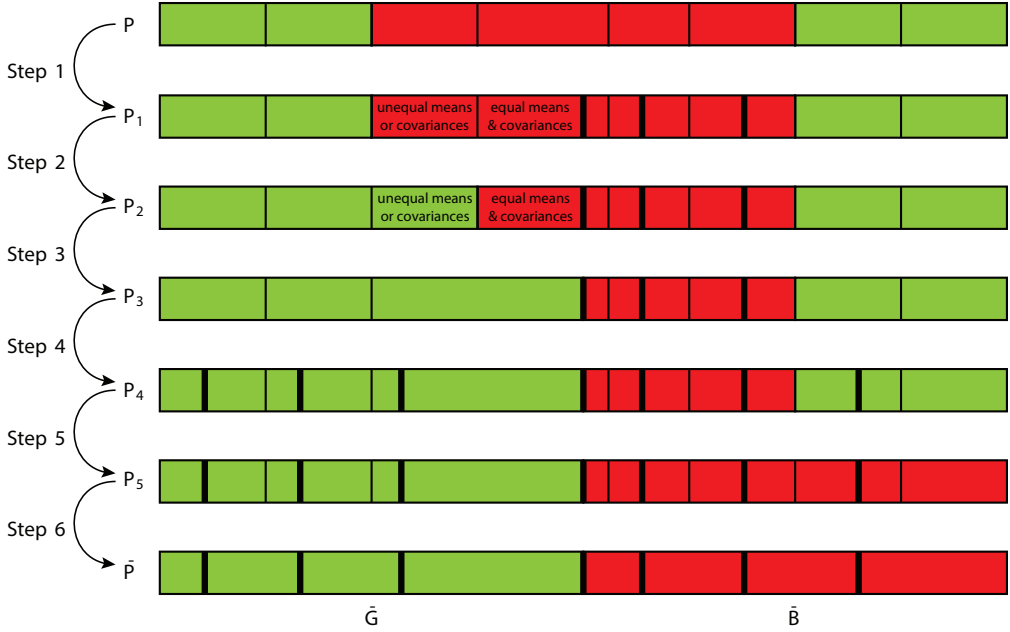}
		\caption{An example of the sequence of partitions constructed in the consistency proof for the OBFS algorithm, starting from the partition $P$ (top) and ending at $\bar{P}$ (bottom).  Green represents good blocks in the current partition (thus $\bar{G}$ is the set of green blocks at the bottom), and red represents bad blocks in the current partition (thus $\bar{B}$ is the set of red blocks at the bottom).}
		\label{fig:consistency_optimal_example2}
	\end{figure}
	
	\noindent\textbf{Step 1.} 
	Let $P_1^1, \ldots, P_1^{K_1}$, $K_1 \geq 2$, be a sequence of partitions where (1) we let $P_1^1 = P$, (2) for each $k = 1, \ldots, K_1-1$, we let $P_1^{k+1}$ be identical to $P_1^k$ except with one bad block partitioned into two bad blocks such that one of the smaller blocks is the intersection of the original block with a bad block in $\bar{P}$, and (3) we iterate until no more blocks can be split.  
	The order we split blocks does not matter; at the end we always obtain the partition $P_1 = P_1^{K_1}$ that is identical to $P$ except with each bad block split into smaller blocks that are each either a subset of a bad block in $\bar{P}$ or a subset of the unambiguous set of good features $\bar{G}$.  
	Suppose $P \neq P_1$.  
	Note $P_1$ is a strict refinement of $P$.
	By Lemma~\ref{lem:ratio_q},
	w.p.~$1$ there exist $s_1,N_1>0$ (which may depend on the sample) such that $q(P)/q(P_1) < n^{-s_1}$ for all $n > N_1$.  
	Furthermore,
	\begin{align}
	\frac{a(P)}{a(P_1)} = \prod_{k=1}^{K_1-1} \frac{a(P_1^k)}{a(P_1^{k+1})}.  
	\label{eq:ratio_a_indep_labeled_dep_bad}
	\end{align}
	For each $k = 1, \ldots, K_1-1$, by Lemma~\ref{lem:ratio_a_indep_labeled_dep_bad}, w.p.~$1$ there exist $c_1^k, N_1^k > 0$ (which may depend on the sample) such that $a(P_1^k)/a(P_1^{k+1}) < (\log n)^{c_1^k}$ for all $n > N_1^k$. 
	By~\eqref{eq:ratio_a_indep_labeled_dep_bad}, w.p.~$1$ there exist $s_1,c_1,N>0$ (namely $s_1$ from above, $c_1 = c_1^1 + \ldots + c_1^{K_1-1}$, and $N = \max(N_1, N_1^1, \ldots, N_1^{K_1-1})$) such that 
	$(q(P)/q(P_1)) \times (a(P)/a(P_1)) < n^{-s_1} (\log n)^{c_1}$
	for all $n > N$.  
	Finally, since $n^{-s_1}$ for $s_1 > 0$ dominates $(\log n)^{c_1}$, $(q(P)/q(P_1)) \times (a(P)/a(P_1)) \to 0$ w.p.~$1$.  
	
	\noindent\textbf{Step 2.} 
	Construct a sequence of partitions starting from $P_1$ where, in each step, one bad block in the current partition that is (a) contained in $\bar{G}$, and (b) has either different means or different covariances between the classes, is labeled good in the next partition, and iterate until no more blocks can be re-labeled.  
	The order we re-label blocks does not matter; we always obtain the same final partition, which we call $P_2$.  
	Suppose $P_1 \neq P_2$.  
	By Lemma~\ref{lem:ratio_q}, w.p.~$1$ there exists $s_2$ such that $q(P_1)/q(P_2) < n^{-s_2}$ for all $n$ large enough.  
	By Lemma~\ref{lem:ratio_a_good_labeled_bad_different}, w.p.~$1$ there exists $0 < r_2 < 1$ such that $a(P_1)/a(P_2) < r_2^n$ for all $n$ large enough.  
	Since $r_2^n$ for $0 < r_2 < 1$ dominates $n^{-s_2}$, $q(P_1)/q(P_2) \times a(P_1)/a(P_2) \to 0$ w.p.~$1$.  
	
	\noindent\textbf{Step 3.} 
	Some bad blocks in $P_2$ may have equal means and covariances between the classes, but may still be contained in $\bar{G}$ because they are correlated with, or connected through a chain of correlation with, features in $\bar{G}$ that have either different means or different covariances between the classes.  Here we construct a sequence of partitions to correct the label of these features.  
	The construction goes as follows: starting from $P_2$, in each step merge one bad block in the current partition that is contained in $\bar{G}$ (which must have the same means and covariances between classes after Step~2) with one good block in the current partition that is correlated in at least one class (because it is correlated it must also be in $\bar{G}$), label the resulting block as good in the next partition, and iterate until no more blocks can be merged.  
	
	First only blocks directly correlated with good blocks in the original partition can be moved, then only blocks directly correlated with good blocks in the original partition or the block just moved can be moved, etc.  All bad blocks in $P_2$ that are actually in $\bar{G}$ will eventually be moved because: (1) all features that have either different means or different variances have already been correctly labeled good in Step 2, 
	and (2) the definition of $\bar{P}$ guarantees that every feature in $\bar{G}$ is connected to at least one feature in $\bar{G}$ with either different means or different variances via at least one chain of pairwise-correlated features that are all in the same good block in $\bar{P}$.    
	Thus, each bad block in the final partition, call it $P_3$, is guaranteed to be a subset of a bad block in $\bar{P}$.  
	
	Suppose $P_2 \neq P_3$.  
	By Lemma~\ref{lem:ratio_q}, w.p.~$1$ there exists $s_3$ such that $q(P_2)/q(P_3) < n^{-s_3}$ for all $n$ large enough.  
	By Lemma~\ref{lem:ratio_a_good_labeled_bad_merge_same}, w.p.~$1$ there exists $0 < r_3 < 1$ such that $a(P_2)/a(P_3) < r_3^n$ for all $n$ large enough.  
	Since $r_3^n$ for $0 < r_3 < 1$ dominates $n^{-s_3}$, $q(P_2)/q(P_3) \times a(P_2)/a(P_3) \to 0$ w.p.~$1$.  
	
	\noindent\textbf{Step 4.} 
	Construct a sequence of partitions starting from $P_3$ where, in each step, one good block in the current partition is split into two good blocks such that one of the smaller blocks is the intersection of the original block with a (good or bad) block in $\bar{P}$, and iterate until no more blocks can be split.  
	The order we split blocks does not matter;
	in the final partition, call it $P_4$, each bad block is a subset of a bad block in $\bar{P}$ and each good block is a subset of a good or bad block in $\bar{P}$.  
	Suppose $P_3 \neq P_4$.  
	Note that $P_4$ is a strict refinement of $P_3$.
	By Lemma~\ref{lem:ratio_q}, w.p.~$1$ there exists $s_4>0$ such that $q(P_3)/q(P_4) < n^{-s_4}$ for all $n$ large enough.  
	By Lemma~\ref{lem:ratio_a_indep_labeled_dep_good}, w.p.~$1$ there exists $c_4 > 0$ such that $a(P_3)/a(P_4) < (\log n)^{c_4}$ for all $n$ large enough.  
	Finally, since $n^{-s_4}$ for $s_4 > 0$ dominates $(\log n)^{c_4}$, $q(P_3)/q(P_4) \times a(P_3)/a(P_4) \to 0$ w.p.~$1$.  
	
	\noindent\textbf{Step 5.} 
	Construct a sequence of partitions starting from $P_4$ where, in each step, one good block in the current partition that is contained in a bad block in $\bar{P}$ is labeled bad in the next partition, and iterate until no more blocks can be re-labeled.  
	The order we re-label blocks does not matter; in the final partition, call it $P_5$, each block is a subset of a block in $\bar{P}$ of the same type.  
	Suppose $P_4 \neq P_5$.  
	Note that $P_5$ is a mesh of $P_4$, every good block in $P_5$ is a good block in $P_4$, and every bad block in $P_5$ is a block (good or bad) in $P_4$.  
	By Lemma~\ref{lem:ratio_q}, w.p.~$1$ there exists $s_5>0$ such that $q(P_4)/q(P_5) < n^{-s_5}$ for all $n$ large enough.  
	By Lemma~\ref{lem:ratio_a_bad_labeled_good}, w.p.~$1$ there exists $c_5 > 0$ such that $a(P_4)/a(P_5) < (\log n)^{c_5}$ for all $n$ large enough.  
	Since $n^{-s_5}$ for $s_5 > 0$ dominates $(\log n)^{c_5}$, $q(P_4)/q(P_5) \times a(P_4)/a(P_5) \to 0$ w.p.~$1$.  
	
	\noindent\textbf{Step 6.}
	Construct a sequence of partitions from $P_5$ to $\bar{P}$ where each partition is a strict refinement of $\bar{P}$ and, in each step, two blocks of the same type (good or bad) in the current partition are merged into a larger block of the same type in the next partition that is contained in a block of the same type in $\bar{P}$.  The order we merge blocks does not matter, as long as the pair of blocks merged in each step are correlated in at least one class.  
	Suppose $P_5 \neq \bar{P}$.  
	By Lemma~\ref{lem:ratio_q}, w.p.~$1$ there exists $s_6$ such that $q(P_5)/q(\bar{P}) < n^{-s_6}$ for all $n$ large enough.  
	By Lemmas~\ref{lem:ratio_a_dep_labeled_indep_good} and~\ref{lem:ratio_a_dep_labeled_indep_bad}, w.p.~$1$ there exists $0 < r_6 < 1$ such that $a(P_5)/a(\bar{P}) < r_6^n$ for all $n$ large enough.  
	Since $r_6^n$ for $0 < r_6 < 1$ dominates $n^{-s_6}$, $q(P_5)/q(\bar{P}) \times a(P_5)/a(\bar{P}) \to 0$ w.p.~$1$.  
	
	\noindent\textbf{Step 7.}
	We have that 
	\begin{equation}
	\frac{q(P)a(P)}{q(\bar{P})a(\bar{P})} 
	= 
	\frac{q(P)a(P)}{q(P_1)a(P_1)}	\cdot 
	\frac{q(P_1)a(P_1)}{q(P_2)a(P_2)}	\cdot 
	\frac{q(P_2)a(P_2)}{q(P_3)a(P_3)}	\cdot 
	\frac{q(P_3)a(P_3)}{q(P_4)a(P_4)}	\cdot 
	\frac{q(P_4)a(P_4)}{q(P_5)a(P_5)}	\cdot 
	\frac{q(P_5)a(P_5)}{q(\bar{P})a(\bar{P})}.  
	\label{eq:hokm_thm_breakdown}
	\end{equation}
	Since $P \neq \bar{P}$, at least one of Steps~1-6 applies.  For the steps that apply, the corresponding ratio in~\eqref{eq:hokm_thm_breakdown} 
	converges to~$0$ w.p.~$1$.  For the steps that do not apply, the corresponding ratio equals~$1$.  Thus, \eqref{eq:hokm_thm} holds.  
\end{proof}

\begin{mylem}
	\label{lem:ratio_q}
	\fussy
	Let $P = ( \{G_1, \ldots, G_u\}, \{B_1, \ldots, B_v\} )$ and $P' = ( \{G_1', \ldots, G_{u'}'\}, \{B_1', \ldots, B_{v'}'\} )$ be arbitrary feature partitions, and let $S_{\infty}$ be a fixed sample such that $\liminf_{n \to \infty} \rho_n >0$ and $\limsup_{n \to \infty} \rho_n <1$.  Then there exists $s \in \mathbb{R}$ and $N > 0$ such that 
	${q(P)}/{q(P')} < n^{-s}$ for all $n > N$.  
	Further, this holds for $s > 0$ if: (i) $P'$ is a strict refinement of $P$, or (ii) $P'$ is a mesh of $P$ where $P \neq P'$, every good block in $P'$ is a subset of a good block in $P$, and every bad block in $P'$ is a subset of a good or bad block in $P$.  
	\sloppy
\end{mylem}

\begin{proof}
	For any feature partition $P$, we may write	
	\begin{align}
	q(P) 
	&= c^P_1 \prod_{\substack{i=1, \ldots, u, \\ y = 0,1}} 
	\left( \nu^{G_i*}_y \right)^{-0.5|G_i|}
	\Gamma_{|G_i|}(0.5\kappa^{G_i*}_y) 
	(\kappa^{G_i*}_y-|G_i|-1)^{-0.5 |G_i| \kappa^{G_i*}_y} 
	\nonumber \\ 
	&\quad \times
	\prod_{j=1, \ldots, v}
	\left( \nu^{B_j*} \right)^{-0.5|B_j|}
	\Gamma_{|B_j|}(0.5\kappa^{B_j*}) 
	(\kappa^{B_j*}-|B_j|-1)^{-0.5 |B_j| \kappa^{B_j*}}, 
	\label{eq:q_of_P}
	\end{align}
	where $c^P_1$ is a positive constant that does not depend on $n$.  
	We first focus on each block individually.  
	$\nu^{G_i*}_y$ and $\kappa^{G_i*}_y$ are asymptotically equivalent to $n_y$, and $\nu^{B_j*}$ and $\kappa^{B_j*}$ are asymptotically equivalent to $n$.  
	Further, for all positive integers $d$, 
	\begin{equation}
	\Gamma_d(x) = \pi^{d(d-1)/4} \prod_{k=1}^d \Gamma\left( x + (1-k)/2 \right), 
	\label{eq:gammadequiv}
	\end{equation}
	and by Stirling's formula, 
	\begin{equation}
	\Gamma_d(x) \sim \pi^{d(d-1)/4} \prod_{k=1}^d 
	\sqrt{\frac{2\pi}{x + (1-k)/2}}\left(\frac{x + (1-k)/2}{e}\right)^{x + (1-k)/2}
	\label{eq:gammadequiv2}
	\end{equation}
	as $x \to \infty$.  
	Thus, for all $A \subseteq F$, 
	\begin{align}
	&\left( \nu^{A*}_y \right)^{-0.5|A|}
	\Gamma_{|A|}(0.5\kappa^{A*}_y) 
	(\kappa^{A*}_y-|A|-1)^{-0.5 |A| \kappa^{A*}_y} \\
	&\qquad \sim
	c_2^A n_y^{-0.75|A| - 0.25|A|^2} (2e)^{-0.5|A|n_y} 
	\prod_{k = 1}^{|A|} \left(\frac{n_y + \kappa_y^A + 1 - k}{n_y + \kappa_y^A - |A| - 1}\right)^{0.5 n_y}
	\end{align}
	as $n_y \to \infty$, where $c_2^A>0$ does not depend on $n$.  
	For any constants $c_1$ and $c_2$, 
	\begin{equation}
	(x-c_1)^{-c_2 x} \sim e^{c_1 c_2} x^{-c_2 x}
	\label{eq:sim_property}
	\end{equation}
	as $x \to \infty$.  Hence, 
	\begin{align}
	\left( \nu^{A*}_y \right)^{-0.5|A|}
	\Gamma_{|A|}(0.5\kappa^{A*}_y) 
	(\kappa^{A*}_y-|A|-1)^{-0.5 |A| \kappa^{A*}_y} 
	&\sim
	c_3^A n_y^{-0.75|A| - 0.25|A|^2} (2e)^{-0.5|A|n_y}
	\label{eq:good_blocks}
	\end{align}
	as $n_y \to \infty$, where $c_3^A>0$ does not depend on $n$.  
	Similarly, 
	\begin{align}
	\left( \nu^{A*} \right)^{-0.5|A|}
	\Gamma_{|A|}(0.5\kappa^{A*}) 
	(\kappa^{A*}-|A|-1)^{-0.5 |A| \kappa^{A*}} 
	&\sim
	c_4^A n^{-0.75|A| - 0.25|A|^2} (2e)^{-0.5|A|n}
	\label{eq:bad_blocks}
	\end{align}
	as $n \to \infty$, where $c_4^A>0$ does not depend on $n$.  Applying~\eqref{eq:good_blocks} and~\eqref{eq:bad_blocks} in~\eqref{eq:q_of_P}, we have that 
	$q(P) \sim c^P_5 
	(n_0 n_1)^{-r_1^P}
	n^{-r_2^P} 
	(2e)^{-0.5 |F| n}$ as $n \to \infty$, 
	where $c^P_5 > 0$ does not depend on $n$, $r_1^P \equiv 0.75|G| +0.25 \sum_{i=1}^{u} |G_i|^2$, $r_2^P \equiv 0.75|B| + 0.25\sum_{j=1}^{v} |B_j|^2$, and we treat $n_0$ and $n_1$ as functions of $n$.  Applying this in the ratio, we have: 
	\begin{align}
	\frac{q(P)}{q(P')} 
	&\sim
	c (\rho_n (1-\rho_n))^{-r_1} n^{-r_2}
	\end{align}
	as $n \to \infty$, 
	where $c \equiv c^P_3 / c^{P'}_3$, $r_1 \equiv r_1^{P'} - r_1^P$ and $r_2 \equiv 2 r_1^{P'} - 2 r_1^{P} + r_2^{P'} - r_2^{P}$.  
	We always have $\limsup_{n \to \infty} \rho_n (1-\rho_n) \leq 0.25$, and by our assumed bounds on the limit inferior and limit superior of $\rho_n$ we also have $\liminf_{n \to \infty} \rho_n (1-\rho_n) > 0$.  Thus, for all $s < r_2$,
	\begin{align}
	\lim_{n \to \infty} \frac{q(P)/q(P')}{n^{-s}}
	&= \lim_{n \to \infty} \frac{c(\rho_n (1-\rho_n))^{-r_1} n^{-r_2}}{n^{-s}} 
	= 0.
	\label{eq:final_lim}
	\end{align}
	Further, 
	\begin{align}
	r_2
	&= \frac{1}{4}\left(
	3|G| - 3|G'| 
	+ 2\sum_{i=1}^{u} |G_i|^2 
	- 2\sum_{i=1}^{u'} |G_i'|^2 
	+ \sum_{j=1}^{v} |B_j|^2 
	- \sum_{j=1}^{v'} |B_j'|^2
	\right).
	\end{align}
	If $P'$ is a strict refinement of $P$, then $r_2>0$.  
	Also, if $P'$ is a mesh of $P$ where $P \neq P'$, then every good block in $P'$ is a subset of a good block in $P$ and every bad block in $P'$ is a subset of a good or bad block in $P$, thus $r_2 > 0$.  
	If $r_2 > 0$, then \eqref{eq:final_lim} holds for all $0 < s < r_2$.  
\end{proof}

\begin{mylem}
	\label{lem:convergence_rate_of_C}
	Let $A \subseteq F$ be any non-empty feature set such that $\Sigma_0^A$ and $\Sigma_1^A$ are full rank.  
	Suppose for all $f \in F$ the fourth order moment across both classes exists and is finite. Let $S_{\infty}$ be a balanced sample.  Then, w.p.~$1$ there exist $K, N>0$ (which may depend on the sample) such that
	\begin{align}
	|\hat{\mu}_y^A(i) - \mu_y^A(i)| &< K \sqrt{\frac{\log \log n}{n}}, 
	\label{eq:bound_means} \\
	\left| \frac{|C_y^A|}{|\hat{\Sigma}_y^A|} - 1\right| &< \frac{K}{n}, 
	\label{eq:bound_Cy_ratio}\\
	\left| \hat{\Sigma}_y^A(i,j) - \Sigma_y^A(i,j)\right| &< K \sqrt{\frac{\log \log n}{n}},
	\label{eq:bound_covar_hat} \\
	\left| \frac{|\hat{\Sigma}_y^A|}{|\Sigma_y^A|} - 1\right| &< K \sqrt{\frac{\log \log n}{n}}, 
	\label{eq:bound_covar_hat_ratio} \\
	\left| C_y^A(i,j) - \Sigma_y^A(i,j)\right| &< K \sqrt{\frac{\log \log n}{n}},
	\label{eq:CSylog_points} \\
	\left| \frac{|C^A|}{|\hat{\Sigma}^A|} - 1\right| &< \frac{K}{n}, 
	\label{eq:bound_C_ratio} \\
	\left| \frac{|\hat{\Sigma}^A|}{|\Sigma_n^A|} - 1\right| &< K \sqrt{\frac{\log \log n}{n}}, 
	\label{eq:bound_covar_hat_ratio2} \\
	\left| C^A(i,j) - \Sigma^A_n(i,j)\right| &< K \sqrt{\frac{\log \log n}{n}}
	\label{eq:CSlog_points}
	\end{align}
	all hold for all $n > N$, $y=0,1$ and $i,j = 1, \ldots, |A|$, where 
	\begin{align}
	\Sigma_{n}^A 
	&= \rho_n \Sigma_0^A
	+ \left(1-\rho_n\right) \Sigma_1^A 
	+ \rho_n (1-\rho_n)
	(\mu_0^A - \mu_1^A)(\mu_0^A - \mu_1^A)^T.  
	\label{eq:Sigma_limit_mixed_class}	
	\end{align}
\end{mylem}

\begin{proof}
	Fix $A$.  Since $S_{\infty}$ is a balanced sample, $n_y \to \infty$ as $n \to \infty$, and by the strong law of large numbers $\hat{\mu}_y^A$ converges to $\mu_y^A$ and $\hat{\Sigma}_y^A$ converges to $\Sigma_y^A$ w.p.~$1$ for both $y=0, 1$.  In the rest of this proof, we only consider the event where $\hat{\mu}_y^A$ and $\hat{\Sigma}_y^A$ converge.  
	
	Since $S_{\infty}$ is a balanced sample, there exist $0 < p_0, p_1 < 1$  and $N_1>0$ such that $n_y/n > p_y$ for all $n > N_1$ and $y = 0,1$.  
	By the law of the iterated logarithm~\citep{hartman1941law}, w.p.~$1$ there exist $K_1>0$ and $N_2>N_1$ such that
	\begin{align}
	|\hat{\mu}_y^A(i) - \mu_y^A(i)|
	&< K_1\sqrt{\frac{\log \log n_y}{n_y}} \notag \\
	&< \frac{K_1}{\sqrt{p_y}} \sqrt{\frac{\log \log n}{n}}
	\end{align}
	for all $n>N_2$ and $i = 1, \ldots, |A|$~\citep{igib1}.  This proves~\eqref{eq:bound_means}.  
	
	We can decompose the sample covariance for class $y$ as follows:
	\begin{align}
	&\hat{\Sigma}_y^A(i,j) 
	= \frac{1}{n_y-1} \sum_{k = 1}^{n_y} (x_{y,k}^i - \hat{\mu}_y^A(i))(x_{y,k}^j - \hat{\mu}_y^A(j)) \notag \\
	&= \frac{1}{n_y-1} \sum_{k = 1}^{n_y} (x_{y,k}^i - \mu_y^A(i))(x_{y,k}^j - \mu_y^A(j)) 
	- \frac{n_y}{n_y-1} (\hat{\mu}_y^A(i) - \mu_y^A(i))(\hat{\mu}_y^A(j) - \mu_y^A(j))
	\label{eq:covar_bound_decompose}
	\end{align}
	where $x_{y,k}^i$ is the $k$th sample of feature $i$ in class $y$.  
	By~\eqref{eq:bound_means}, w.p.~$1$ there exists $K_2>0$ and $N_3>N_2$ such that 
	\begin{align}
	\left|\frac{n_y}{n_y-1} (\hat{\mu}_y^A(i) - \mu_y^A(i))(\hat{\mu}_y^A(j) - \mu_y^A(j))\right|
	&< K_2 \frac{\log \log n}{n}
	\label{eq:covar_bound_partial}
	\end{align}
	for all $n>N_3$ and $i,j = 1, \ldots, |A|$.  
	$(x_k^i - \mu_y^A(i))(x_k^j - \mu_y^A(j))$ is a random variable with mean $\Sigma_y^A(i,j)$ and finite variance $V_y(i,j)$.  
	Suppose $V_y(i,j) > 0$.  
	By the law of the iterated logarithm, w.p.~$1$ 
	\begin{align}
	\limsup_{n_y \to \infty} \frac{1}{\sqrt{2 n_y \log \log n_y}} \sum_{k = 1}^{n_y} \frac{(x_{y,k}^i - \mu_y^A(i))(x_{y,k}^j - \mu_y^A(j)) - \Sigma_y^A(i,j)}{\sqrt{V_y(i,j)}} = 1
	\end{align}
	for all $i,j = 1, \ldots, |A|$.  Thus, w.p.~$1$ there exists $K_3>0$ and $N_4>N_3$ such that
	\begin{align}
	&\left|\frac{1}{n_y-1}\sum_{k = 1}^{n_y} (x_{y,k}^i - \mu_y^A(i))(x_{y,k}^j - \mu_y^A(j)) 
	- \Sigma_y^A(i,j) \right| \notag \\
	&\qquad\leq \left|\frac{1}{n_y-1}\sum_{k = 1}^{n_y} (x_{y,k}^i - \mu_y^A(i))(x_{y,k}^j - \mu_y^A(j)) 
	- \frac{n_y}{n_y-1} \Sigma_y^A(i,j)\right| + \frac{1}{n_y-1} \Sigma_y^A(i,j) \notag \\
	&\qquad <  \frac{2 \sqrt{V_y(i,j) n_y \log \log n_y}}{n_y-1} 
	+ \frac{1}{n_y-1} \Sigma_y^A(i,j) \notag \\
	&\qquad < K_3 \sqrt{\frac{\log \log n}{n}}
	\label{eq:covar_bound_sum}
	\end{align}
	for all $n>N_4$ and $i,j = 1, \ldots, |A|$. A similar inequality holds when $V_y(i,j) = 0$. Combining~\eqref{eq:covar_bound_decompose}, \eqref{eq:covar_bound_partial} and~\eqref{eq:covar_bound_sum}, w.p.~$1$ there exists $K_4>0$ such that
	\begin{align}
	|\hat{\Sigma}_y^A(i,j) - \Sigma_y^A(i,j)|
	&< K_3 \sqrt{\frac{\log \log n}{n}} + K_2 \frac{\log \log n}{n} \notag \\
	&< K_4 \sqrt{\frac{\log \log n}{n}}
	\end{align}
	for all $n>N_4$ and $i,j = 1, \ldots, |A|$.  Thus~\eqref{eq:bound_covar_hat} holds.  
	Since $|\hat{\Sigma}_y^A|$ is a polynomial function of the $\hat{\Sigma}_y^A(i,j)$, where each term has degree $|A|$ and coefficient $\pm1$, w.p.~$1$ there exists $K_5>0$ such that 
	\begin{align}
	||\hat{\Sigma}_y^A| - |\Sigma_y^A|| < K_5 \sqrt{\frac{\log\log n}{n}}
	\label{eq:hat_Sigmay_det}
	\end{align}
	for all $n>N_4$.  Dividing both sides by $|\Sigma_y^A|$ proves~\eqref{eq:bound_covar_hat_ratio}.  
	
	Note that, 
	\begin{align}
	C^A_y&=
	\frac{n_y-1}{n_y+\kappa^{A}_y-|A|-1} \hat{\Sigma}^{A}_y
	+ \frac{1}{n_y+\kappa^{A}_y-|A|-1} S^{A}_y \notag \\
	&\quad + \frac{\nu^{A}_y n_y}{(n_y+\kappa^{A}_y-|A|-1)(n_y + \nu^{A}_y)} (\hat{\mu}^{A}_y-m^{A}_y)(\hat{\mu}^{A}_y-m^{A}_y)^T.
	\label{eq:C_class_y}
	\end{align}
	Further~\eqref{eq:bound_means} implies that w.p.~$1$ $\hat{\mu}_y^A(i) - m_y^A(i)$ is bounded for all $n>N_2$.  Similarly, \eqref{eq:bound_covar_hat} implies that w.p.~$1$ $\hat{\Sigma}_y^A(i,j)$ is bounded for all $n>N_3$.  
	Thus, from~\eqref{eq:C_class_y}, w.p.~$1$ there exists $K_6 > 0$ and $N_5>N_4$ such that
	\begin{align}
	|C_y^A(i,j) - \hat{\Sigma}_y^A(i,j)|
	&\leq \frac{K_6}{n}
	\label{eq:bound_Cy}
	\end{align}
	for all $n > N_5$ and $i,j = 1, \ldots |A|$.  
	Again noting that the determinant is a polynomial, w.p.~$1$ there exists $K_7 > 0$ such that 
	\begin{align}
	||C_y^A| - |\hat{\Sigma}_y^A|| < \frac{K_7}{n}
	\end{align}
	for all $n > N_5$.  By~\eqref{eq:hat_Sigmay_det}, $|\hat{\Sigma}_y^A|$ is bounded for all $n>N_4$ w.p.~$1$.  Thus, \eqref{eq:bound_Cy_ratio} holds.  
	Applying the triangle inequality on $C_y^A(i,j) - \Sigma_y^A(i,j)$ and applying~\eqref{eq:bound_covar_hat} and~\eqref{eq:bound_Cy}, we have that~\eqref{eq:CSylog_points} also holds.  
	
	The sample covariance across all samples in both classes can be decomposed as
	\begin{align}
	\hat{\Sigma}^A 
	&= \left(\rho_n - \frac{1-\rho_n}{n-1}\right) \hat{\Sigma}_0^A
	+ \left(1-\rho_n - \frac{\rho_n}{n-1}\right) \hat{\Sigma}_1^A \notag \\
	&\quad + \frac{\rho_n (1-\rho_n) n}{n-1} 
	(\hat{\mu}_0^A - \hat{\mu}_1^A)(\hat{\mu}_0^A - \hat{\mu}_1^A)^T.  
	\label{eq:Sigma_mixed_class}	
	\end{align}
	Define
	\begin{align}
	\hat{\Sigma}_n^A 
	&= \rho_n \hat{\Sigma}_0^A
	+ (1-\rho_n) \hat{\Sigma}_1^A 
	+ \rho_n (1-\rho_n) (\hat{\mu}_0^A - \hat{\mu}_1^A)(\hat{\mu}_0^A - \hat{\mu}_1^A)^T.  
	\label{eq:Sigma_n_mixed_class}	
	\end{align}
	Again since w.p.~$1$ $\hat{\mu}_y^A(i)$ and $\hat{\Sigma}_y^A(i,j)$ are bounded for all $n>N_3$, by the triangle inequality w.p.~$1$ there exists $K_8 > 0$ such that
	\begin{align}
	|\hat{\Sigma}^A(i,j)-\hat{\Sigma}^A_n(i,j)|
	&\leq \frac{1-\rho_n}{n-1} \hat{\Sigma}_0^A(i,j)
	+ \frac{\rho_n}{n-1} \hat{\Sigma}_1^A(i,j) \notag \\
	&\quad + \frac{\rho_n (1-\rho_n)}{n-1} 
	|\hat{\mu}_0^A(i) - \hat{\mu}_1^A(i)| 
	|\hat{\mu}_0^A(j) - \hat{\mu}_1^A(j)| \notag \\
	&<\frac{K_8}{n}
	\label{eq:hShSn}
	\end{align}
	for all $n>N_3$ and $i,j =1,\ldots, |A|$.  
	Further, by the triangle inequality
	\begin{align}
	|\hat{\Sigma}^A_n(i,j)-\Sigma^A_n(i,j)|
	&\leq \rho_n |\hat{\Sigma}_0^A(i,j) - \Sigma_0^A(i,j)|
	+ (1-\rho_n) |\hat{\Sigma}_1^A(i,j) - \Sigma_1^A(i,j)| \notag \\
	&\quad + \rho_n (1-\rho_n) d(i,j),
	\label{eq:hatSigman_Sigman}
	\end{align}
	for all $n> N_3$ and $i,j =1,\ldots, |A|$ w.p.~$1$, where 
	\begin{align}
	d(i,j) &= \left| (\hat{\mu}_0^A(i) - \hat{\mu}_1^A(i))
	(\hat{\mu}_0^A(j) - \hat{\mu}_1^A(j))
	- (\mu_0^A(i) - \mu_1^A(i))
	(\mu_0^A(j) - \mu_1^A(j)) \right| \notag \\
	&= \left| (d_0(i) + d_{01}(i) + d_1(i))
	(d_0(j) + d_{01}(j) + d_1(j))
	- d_{01}(i)d_{01}(j) \right| \notag \\
	&\leq 
	|d_0(i)d_0(j)|
	+|d_0(i)d_{01}(j)|
	+|d_0(i)d_1(j)|
	+|d_{01}(i)d_0(j)|
	+|d_{01}(i)d_1(j)| \notag \\
	&\quad +|d_1(i)d_0(j)|
	+|d_1(i)d_{01}(j)|
	+|d_1(i)d_1(j) |,
	\end{align}
	$d_0 = \hat{\mu}_0^A - \mu_0^A$, $d_{01} = \mu_0^A - \mu_1^A$ and $d_1 = \mu_1^A - \hat{\mu}_1^A$.  
	$d_{01}(i)$ is constant for all $i=1, \ldots, |A|$.  By~\eqref{eq:bound_means}, w.p.~$1$ there exists $K_9, N_6>N_5$ such that $|d_y(i)| < K_9 \sqrt{\log\log n / n}$ for all $n>N_6$, $y=0,1$ and $i = 1, \ldots, |A|$. Thus, w.p.~$1$ there exists $K_{10}>0$ such that $d(i,j) < K_{10} \sqrt{\log\log n / n}$ for all $n>N_6$ and $i,j = 1, \ldots, |A|$.  Applying this fact and~\eqref{eq:bound_covar_hat} to~\eqref{eq:hShSn}, w.p.~$1$ there exists $K_{11} > 0$ and $N_7>N_6$ such that 
	\begin{align}
	|\hat{\Sigma}_n^A(i,j) - \Sigma_n^A(i,j)| < K_{11} \sqrt{\frac{\log\log n}{n}}
	\end{align}
	for all $n>N_7$ and $i,j =1,\ldots, |A|$.  Combining this fact with~\eqref{eq:hatSigman_Sigman}, w.p.~$1$ there exists $K_{12}>0$ such that 
	\begin{align}
	|\hat{\Sigma}^A(i,j) - \Sigma_n^A(i,j)| < K_{12} \sqrt{\frac{\log\log n}{n}}
	\label{eq:hatSigma_bound}
	\end{align}
	for all $n>N_7$ and $i,j =1,\ldots, |A|$.  
	Again since the determinant is a polynomial, w.p.~$1$ there exists $K_{13}$ such that 
	\begin{align}
	||\hat{\Sigma}^A| - |\Sigma_n^A|| < K_{13} \sqrt{\frac{\log\log n}{n}}
	\label{eq:hatSigma_det}
	\end{align}
	for all $n>N_7$.  
	Observe from~\eqref{eq:Sigma_limit_mixed_class} that $\Sigma_{n}^A(i,j)$ is lower bounded by 
	\begin{equation}
	\min(\Sigma_0^A(i,j), \Sigma_1^A(i,j)) + \min(0, 0.25 (\mu_0^A(i) - \mu_1^A(i))(\mu_0^A(j) - \mu_1^A(j))),
	\end{equation}
	and upper bounded by 
	\begin{equation}
	\max(\Sigma_0^A(i,j), \Sigma_1^A(i,j)) + \max(0, 0.25 (\mu_0^A(i) - \mu_1^A(i))(\mu_0^A(j) - \mu_1^A(j))).
	\end{equation}
	Since $|\Sigma_n^A|$ is a polynomial function of the $\Sigma_n^A(i,j)$, where each term has degree $|A|$ and coefficient $\pm1$, $|\Sigma_n^A|$ must also be upper and lower bounded when $n$ is large enough.  
	\eqref{eq:bound_covar_hat_ratio2} follows by dividing both sides of~\eqref{eq:hatSigma_det} by $|\Sigma_n^A|$, and applying a bound on $|\Sigma_n^A|$ on the right hand side.  
	
	Observe that
	\begin{align}
	C^A&=
	\frac{n-1}{n+\kappa^{A}-|A|-1} \hat{\Sigma}^{A}
	+ \frac{1}{n+\kappa^{A}-|A|-1} S^{A} \notag \\
	&\quad + \frac{\nu^{A} n}{(n+\kappa^{A}-|A|-1)(n + \nu^{A})} (\hat{\mu}^{A}-m^{A})(\hat{\mu}^{A}-m^{A})^T.
	\label{eq:C_mixed_class}
	\end{align}
	Since $\hat{\mu}^{A} = \rho_n \hat{\mu}_0^{A} + (1-\rho_n) \hat{\mu}_1^{A}$, \eqref{eq:bound_means} implies that w.p.~$1$ $\hat{\mu}^A(i) - m^A(i)$ is bounded for all $n>N_2$.  
	Similarly, \eqref{eq:hatSigma_bound} and the fact that $\Sigma_n^A(i,j)$ is bounded for $n$ large enough implies that w.p.~$1$ there exists $N_8>N_7$ such that $\hat{\Sigma}^A(i,j)$ is bounded for all $n>N_8$.  
	Thus, from~\eqref{eq:C_mixed_class}, w.p.~$1$ there exists $K_{14} > 0$ and $N_9>N_8$ such that
	\begin{align}
	|C^A(i,j) - \hat{\Sigma}^A(i,j)|
	&\leq \frac{K_{14}}{n}
	\label{eq:bound_C}
	\end{align}
	for all $n > N_9$ and $i,j = 1, \ldots |A|$.  
	Again since the determinant is a polynomial, w.p.~$1$ there exists $K_{15} > 0$ such that 
	\begin{align}
	||C^A| - |\hat{\Sigma}^A|| < \frac{K_{15}}{n}
	\end{align}
	for all $n > N_9$.  By~\eqref{eq:hatSigma_det} and the fact that $|\Sigma_n^A|$ is bounded for $n$ large enough, w.p.~$1$ there exists $N_{10}>N_9$ such that $|\hat{\Sigma}_y^A|$ is bounded for all $n>N_{10}$.  Thus, \eqref{eq:bound_C_ratio} holds.  
	Applying the triangle inequality on $C^A(i,j) - \Sigma_n^A(i,j)$ and applying~\eqref{eq:hatSigma_bound} and~\eqref{eq:bound_C}, we have that~\eqref{eq:CSlog_points} also holds.  
\end{proof}

\begin{mycor}
	\label{lem:convergence_of_C}
	Let $A \subseteq F$ be any non-empty feature set such that $\Sigma_0^A$ and $\Sigma_1^A$ are full rank.  
	Suppose for all $f \in F$ the fourth order moment across both classes exists and is finite. Let $S_{\infty}$ be a balanced sample.  
	Then, w.p.~$1$ the following all hold for all $y=0,1$ and $i,j=1, \ldots, |A|$, where $\Sigma_n^A$ is given in~\eqref{eq:Sigma_limit_mixed_class}.  
	\begin{enumerate}
		\item $C_y^A(i,j) \to \Sigma_y^A(i,j)$ as $n \to \infty$.
		\item $|C_y^A| \to |\Sigma_y^A|$ as $n \to \infty$.
		\item If $\Sigma^A \equiv \Sigma_0^A = \Sigma_1^A$ and $\mu_0^A = \mu_1^A$, then $C^A(i,j) \to \Sigma^A(i,j)$ as $n \to \infty$.
		\item If $\Sigma^A \equiv \Sigma_0^A = \Sigma_1^A$ and $\mu_0^A = \mu_1^A$, then $|C^A| \to |\Sigma^A|$ as $n \to \infty$.
		\item $C^A(i,j) - \Sigma_n^A(i,j) \to 0$ as $n \to \infty$.
		\item $|C^A| \sim |\Sigma_n^A|$ as $n \to \infty$.
		\item There exist $L,U \in \mathbb{R}$ and $N > 0$ such that $L \leq C^A(i,j) \leq U$ for all $n>N$.
		\item There exist $L,U, N > 0$ such that $L \leq |C^A| \leq U$ for all $n>N$.
	\end{enumerate} 
\end{mycor}

\begin{proof}
	(1) follows from~\eqref{eq:CSylog_points}.
	(2) follows by combining~\eqref{eq:bound_Cy_ratio} and~\eqref{eq:bound_covar_hat_ratio}.  
	(5) follows from~\eqref{eq:CSlog_points}.
	(6) follows by combining~\eqref{eq:bound_C_ratio} and~\eqref{eq:bound_covar_hat_ratio2}.  
	(3) and (4) are special cases of (5) and (6), where in this case $\Sigma_n^A = \Sigma^A$ is constant.  
	Recall that in Lemma~\ref{lem:convergence_rate_of_C} we showed that $\Sigma_{n}^A(i,j)$ and $|\Sigma_n^A|$ are lower and upper bounded when $n$ is large enough.  
	By~\eqref{eq:CSlog_points}, $C^A(i,j)$ must also be upper and lower bounded when $n$ is large enough, so (7) holds.
	Finally, since by item 6 $|C^A|$ and $|\Sigma_n^A|$ are asymptotically equivalent, $|C^A|$ must also be upper and lower bounded when $n$ is large enough, so (8) holds.  
\end{proof}

\begin{mylem}
	\label{lem:ratio_a_indep_labeled_dep_bad}
	Let $A_1$ and $A_2$ be any disjoint feature sets such that $\mu_0^{A_1} = \mu_1^{A_1}$, $\Sigma^{A_1} \equiv \Sigma_0^{A_1} = \Sigma_1^{A_1}$, features in $A_1$ and $A_2$ are uncorrelated in both classes, and $\Sigma_0^{A}$ and $\Sigma_1^{A}$ are full rank, where $A = A_1 \cup A_2$.  
	Suppose for all $f \in F$ the fourth order moment across both classes exists and is finite. 
	Let $S_{\infty}$ be a balanced sample.  Then, w.p.~$1$ there exist $c, N>0$ such that 
	\begin{align}
	\left( \frac
	{|C^{A_1}|^{\kappa^{A_1*}} 
		|C^{A_2}|^{\kappa^{A_2*}}}
	{|C^{A}|^{\kappa^{A*}}}
	\right)^{0.5}
	< (\log n)^c
	\label{eq:ratio_split_bad}
	\end{align}
	for all $n > N$.  
	The left hand side of~\eqref{eq:ratio_split_bad} is $a(P)/a(P')$ when $P$ and $P'$ are identical feature partitions except $A$ is a bad block in $P$, and $A_1$ and $A_2$ are bad blocks in $P'$.  
\end{mylem}

\begin{proof}
	We have that 
	\begin{align}
	\left( \frac
	{|C^{A_1}|^{\kappa^{A_1*}} 
		|C^{A_2}|^{\kappa^{A_2*}}}
	{|C^{A}|^{\kappa^{A*}}}
	\right)^{0.5}
	&= \left( \frac
	{|C^{A_1}|^{\kappa^{A_1}} 
		|C^{A_2}|^{\kappa^{A_2}}}
	{|C^{A}|^{\kappa^{A}}}
	\right)^{0.5}
	\times R_n,
	\label{eq:a_ratio_split_bad}
	\end{align}
	where 
	\begin{align}
	R_n = 
	\left( \frac
	{|C^{A_1}||C^{A_2}|}
	{|C^{A}|}
	\right)^{0.5n}.
	\label{eq:Rn}
	\end{align}
	By Corollary~\ref{lem:convergence_of_C}, the first term in~\eqref{eq:a_ratio_split_bad} is upper bounded by a positive finite constant for $n$ large enough w.p.~$1$.
	Without loss of generality, assume features in $A$ are ordered such that 
	\begin{equation}
	C^A =
	\begin{pmatrix}
	C^{A_1} & V \\
	V^T & C^{A_2}
	\end{pmatrix}
	\end{equation}
	for some matrix $V$. 
	Observe that 
	\begin{align}
	|C^A|=|C^{A_1}| |C^{A_2} - W|,
	\label{eq:decompose_CA}
	\end{align}
	where $W \equiv V^T (C^{A_1})^{-1} V$.  
	Since $\mu_0^{A_1} = \mu_1^{A_1}$, $\Sigma_0^{A_1} = \Sigma_1^{A_1}$, and features in $A_1$ and $A_2$ are uncorrelated in both classes,
	\begin{equation}
	\Sigma^A_n =
	\begin{pmatrix}
	\Sigma^{A_1} & 0_{|A_1|, |A_2|} \\
	0_{|A_2|, |A_1|} & \Sigma^{A_2}_n
	\end{pmatrix},
	\end{equation}
	where $\Sigma^A_n$ and $\Sigma^{A_2}_n$ are defined in~\eqref{eq:Sigma_limit_mixed_class}.
	By Lemma~\ref{lem:convergence_rate_of_C}, w.p.~$1$ there exist $K_1, N_1 > 0$ such that for all $n>N_1$, $i=1,\ldots,|A_1|$ and $j=1,\ldots,|A_2|$, 
	\begin{align}
	\label{eq:Vrate1}
	|V(i,j)|<K_1 \sqrt{\frac{\log \log n}{n}}.
	\end{align}
	Since $W$ is comprised of only quadratic terms in $V$ and $C^{A_1}(i,j) \to \Sigma^{A_1}(i,j)$ w.p.~$1$ (Corollary~\ref{lem:convergence_of_C}), w.p.~$1$ there exist $K_2, N_2 > 0$ such that for all $n>N_2$, $i=1,\ldots,|A_1|$ and $j=1,\ldots,|A_2|$,
	\begin{align}
	|W(i,j)| < K_2 \frac{\log \log n}{n}.
	\end{align}
	Further, since $|C^{A_2}| - |C^{A_2} - W|$ is a polynomial function of the elements of $W$, where each term has a degree between $1$ and $|A_2|$ and a coefficient that is a polynomial function of the elements of $C^{A_2}$ of at most degree $|A_2|-1$, and since $C^{A_2}(i,j)$ is upper and lower bounded for $n$ large enough w.p.~$1$ for all $i,j = 1, \ldots, |A|$ (Corollary~\ref{lem:convergence_of_C}), w.p.~$1$ there exists $K_3, N_3>0$ such that for all $n>N_3$,
	\begin{align}
	\left|  |C^{A_2}| -  |C^{A_2} - W| \right| < K_3 \frac{\log \log n}{n}.  
	\label{eq:CA2_bound}
	\end{align}
	Therefore,
	\begin{align}
	R_n 
	&=
	\left(\frac{|C^{A_1}||C^{A_2}|}{|C^{A_1}| |C^{A_2}- W|}\right)^{0.5n} \nonumber \\ 
	&<
	\left(1-\frac{K_3}{|C^{A_2}|}  \frac{\log \log n}{n}\right)^{-0.5n}
	\end{align}
	for $n>N_3$ w.p.~$1$, where the first line follows from~\eqref{eq:Rn} and~\eqref{eq:decompose_CA} and the second line from~\eqref{eq:CA2_bound}.  Further, by Corollary~\ref{lem:convergence_of_C}, w.p.~$1$ there exist $K_4>0$ and $N_4>N_3$ such that $|C^{A_2}|>K_4$ for all $n>N_4$, thus 
	\begin{align}
	R_n 
	&< 
	\left(1-\frac{K_3}{K_4}  \frac{\log \log n}{n}\right)^{-0.5n} \nonumber \\ 
	&<
	\left(1+\frac{2K_3}{K_4}  \frac{\log \log n}{n}\right)^{0.5n} \nonumber \\ 
	&<
	(\log n)^{K_3/K_4}
	\end{align}
	for all $n > N_4$.  
	The second line holds for $N_4$ large enough that $x = (K_3/ K_4)\log\log n / n$ is between $0$ and $0.5$ for all $n>N_4$, so that $(1-x)^{-1} < 1+2x$.  The last line follows from the fact that $(1+t/x)^x < e^t$ for all $x,t>0$.  
	From~\eqref{eq:a_ratio_split_bad}, the theorem holds with $c=K_5 + {K_3/K_4}$ and $N = N_4$, where $K_5$ is such that $(\log n)^{K_5}$ exceeds a bound on the first term in~\eqref{eq:a_ratio_split_bad}.
\end{proof}

\begin{mylem}
	\label{lem:ratio_a_good_labeled_bad_different}
	Let $G$ be any feature set such that either $\mu_0^G \neq \mu_1^G$ or $\Sigma_0^G \neq \Sigma_1^G$, and $\Sigma_0^G$ and $\Sigma_1^G$ are full rank.  Let $S_{\infty}$ be a balanced sample.  Then, w.p.~$1$ there exist $0 \leq r < 1$ and $N>0$ such that 
	\begin{align}
	\left(
	\frac
	{|C^{G}_0|^{\kappa^{G*}_0} |C^{G}_1|^{\kappa^{G*}_1}}
	{|C^{G}|^{\kappa^{G*}}}
	\right) ^{0.5}
	< r^n
	\label{eq:ratio_correct_good_diff}
	\end{align}
	for all $n > N$.  
	The left hand side of~\eqref{eq:ratio_correct_good_diff} is $a(P)/a(P')$ when $P$ and $P'$ are identical feature partitions except $G$ is a bad block in $P$ and a good block in $P'$.  
\end{mylem}

\begin{proof}
	We have that 
	\begin{align}
	\left(
	\frac
	{|C^{G}_0|^{\kappa^{G*}_0} |C^{G}_1|^{\kappa^{G*}_1}}
	{|C^{G}|^{\kappa^{G*}}}
	\right) ^{0.5}
	&= \left(
	\frac
	{|C^{G}_0|^{\kappa^{G}_0} |C^{G}_1|^{\kappa^{G}_1}}
	{|C^{G}|^{\kappa^{G}}}
	\right) ^{0.5}
	\times R_n^n,
	\label{eq:a_ratio_correct_good_diff}
	\end{align}
	where 
	\begin{align}
	R_n = 
	\left( 
	\frac
	{|C^{G}_0|^{\rho_n} |C^{G}_1|^{1-\rho_n}}
	{|C^{G}|}
	\right)^{0.5}.  
	\label{eq:Rn_lemma4}
	\end{align}
	By Corollary~\ref{lem:convergence_of_C}, the first term in~\eqref{eq:a_ratio_correct_good_diff} is upper bounded by a positive finite constant for $n$ large enough, and
	\begin{align}
	R_n \sim
	\left( 
	\frac
	{|\Sigma^{G}_0|^{\rho_n} |\Sigma^{G}_1|^{1-\rho_n}}
	{|\rho_n \Sigma_0^G + (1-\rho_n) \Sigma_1^G + \rho_n(1-\rho_n)(\mu_0^G - \mu_1^G)(\mu_0^G - \mu_1^G)^T|}
	\right)^{0.5}
	\equiv T_n
	\label{eq:Rn_lemma4_equiv}
	\end{align}
	as $n \to \infty$, both w.p.~$1$.  
	Suppose $\Sigma^G \equiv \Sigma_0^G = \Sigma_1^G$ and $\mu_0^G \neq \mu_1^G$.  Then,
	\begin{align}
	T_n 
	&= \left(
	\frac{|\Sigma^{G}|}
	{|\Sigma^G + \rho_n (1-\rho_n)(\mu_0^G - \mu_1^G)(\mu_0^G - \mu_1^G)^T|}
	\right)^{0.5} \notag \\
	&= \left(
	1 + \rho_n (1-\rho_n)(\mu_0^G - \mu_1^G)^T(\Sigma^G)^{-1}(\mu_0^G - \mu_1^G)
	\right)^{-0.5}.
	\label{eq:T_unequal_mean}
	\end{align}
	The last line follows from the matrix determinant lemma~\citep{harville1998matrix}.  
	Since $S_{\infty}$ is balanced and $\Sigma^G$ is positive-definite, there exists $0 < T < 1$ such that $0 \leq T_n < T$ for all $n$ large enough.  
	The theorem holds for any $T < r < 1$.  
	If $\Sigma_0^G \neq \Sigma_1^G$, then 
	\begin{align}
	T_n \leq \left(
	\frac{|\Sigma_0^{G}|^{\rho_n} |\Sigma_1^{G}|^{1-\rho_n}}
	{|\rho_n \Sigma_0^G
		+ (1-\rho_n) \Sigma_1^G|}
	\right)^{0.5}.
	\label{eq:T_unequal_covar}
	\end{align}
	\cite{fan1950theorem} showed that for any symmetric positive-definite $X$ and $Y$ and $0 \leq p \leq 1$, 
	\begin{equation}
	\label{eq:bm_ineq}
	|p X + (1-p) Y| \geq |X|^{p}  |Y|^{1-p}.
	\end{equation}
	Although not mentioned by Fan, if $0 < p < 1$ then equality holds if and only if $X=Y$ (by the weighted arithmetic-geometric-mean inequality). 
	Suppose $X \neq Y$ and fix $0 < \delta < 0.5$.  Then there exists $\epsilon > 0$ such that $|X|^{p}  |Y|^{1-p} / |p X + (1-p) Y| < 1-\epsilon$ for all $p \in (\delta, 1-\delta)$.  
	Applying this fact here, since $S_{\infty}$ is balanced there exists $0 < T < 1$ such that $0 \leq T_n < T$.  The theorem holds for any $T < r < 1$.  
\end{proof}

\begin{mylem}
	\label{lem:ratio_a_good_labeled_bad_merge_same}
	Let $G_1$ and $G_2$ be any disjoint feature sets such that $\mu_0^{G_2} = \mu_1^{G_2}$, $\Sigma^{G_2} \equiv \Sigma_0^{G_2} = \Sigma_1^{G_2}$, 
	there exists at least one feature in $G_1$ and one feature in $G_2$ that are correlated in at least one class, and $\Sigma_0^{G}$ and $\Sigma_1^{G}$ are full rank, where $G = G_1 \cup G_2$.  
	Let $S_{\infty}$ be a balanced sample.  Then, w.p.~$1$ there exist $0 \leq r < 1$ and $N>0$ such that
	\begin{align}
	\left(
	\frac{|C^{G}_0|^{\kappa^{G*}_0}|C^{G}_1|^{\kappa^{G*}_1}}
	{|C^{G_1}_0|^{\kappa^{G_1*}_0}|C^{G_1}_1|^{\kappa^{G_1*}_1}|C^{G_2}|^{\kappa^{G_2*}}}
	\right)^{0.5}
	< r^n
	\label{eq:ratio_merge_good_diff}
	\end{align}
	for all $n > N$.  The left hand side of~\eqref{eq:ratio_merge_good_diff} is $a(P)/a(P')$ when $P$ and $P'$ are identical feature partitions except $G_1$ is a good block in $P$, $G_2$ is a bad block in $P$, and $G$ is a good block in $P'$.  
\end{mylem}

\begin{proof}
	We have that 
	\begin{align}
	\left(
	\frac{|C^{G}_0|^{\kappa^{G*}_0}|C^{G}_1|^{\kappa^{G*}_1}}
	{|C^{G_1}_0|^{\kappa^{G_1*}_0}|C^{G_1}_1|^{\kappa^{G_1*}_1}|C^{G_2}|^{\kappa^{G_2*}}}
	\right)^{0.5}
	&= \left(
	\frac{|C^{G}_0|^{\kappa^{G}_0}|C^{G}_1|^{\kappa^{G}_1}}
	{|C^{G_1}_0|^{\kappa^{G_1}_0}|C^{G_1}_1|^{\kappa^{G_1}_1}|C^{G_2}|^{\kappa^{G_2}}}
	\right)^{0.5}
	\times  
	R_{0,n}^{n_0}R_{1,n}^{n_1},
	\label{eq:a_ratio_merge_good_diff}
	\end{align}
	where
	\begin{align}
	R_{y,n}
	= \left(\frac{|C^{G}_y|}
	{|C^{G_1}_y| |C^{G_2}|}\right)^{0.5}.
	\end{align}
	By Corollary~\ref{lem:convergence_of_C}, the first term in~\eqref{eq:a_ratio_merge_good_diff} converges to a positive finite constant w.p.~$1$, and 
	\begin{align}
	R_{y,n} \to
	\left( 
	\frac
	{|\Sigma^{G}_y|}
	{|\Sigma^{G_1}_y| |\Sigma^{G_2}_y|}
	\right)^{0.5}
	\equiv R_y
	\label{eq:R_y_merge_good_diff}
	\end{align}
	as $n \to \infty$, both w.p.~$1$.  
	Without loss of generality, assume features in $G$ are ordered such that 
	\begin{equation}
	\Sigma^{G}_y =
	\begin{pmatrix}
	\Sigma^{G_1}_y & V_y \\
	V_y^T & \Sigma^{G_2}_y
	\end{pmatrix}, 
	\end{equation}
	where $V_y$ is a matrix of correlations between features in $G_1$ and $G_2$.  Since 
	\begin{equation}
	|\Sigma_y^{G}| = |\Sigma_y^{G_1}| |\Sigma_y^{G_2} - V_y^T(\Sigma_y^{G_1})^{-1}V_y|,
	\end{equation}
	we have that 
	\begin{align}
	R_{y}
	=
	\left(\frac{
		|\Sigma_y^{G_2} - V_y^T(\Sigma_y^{G_1})^{-1}V_y|}
	{|\Sigma_y^{G_2}|}\right)^{0.5}
	\end{align}
	We always have $|\Sigma_y^{G_2} - V_y^T(\Sigma_y^{G_1})^{-1}V_y| \leq |\Sigma^{G_2}_y|$, and thus $0 \leq R_y \leq 1$.  
	Since there exists at least one feature in $G_1$ and one feature in $G_2$ that are correlated in at least one class, in this class $V_y$ is non-zero, $|\Sigma_y^{G_2} - V_y^T(\Sigma_y^{G_1})^{-1}V_y| < |\Sigma^{G_2}_y|$, and $R_y < 1$.  
	Since $S_{\infty}$ is balanced, there exist $0 < p_0, p_1 < 1$ such that $n_0/n < p_0$ and $n_1/n < p_1$ for $n$ large enough.
	Thus, 
	\begin{align}
	R_{0,n}^{n_0} R_{1,n}^{n_1} \leq 
	R_n^n
	\end{align}
	for $n$ large enough, where $R_n = R_{0,n}^{p_0} R_{1,n}^{p_1}$.  Note $R_n \to R_0^{p_0}R_1^{p_1} \equiv R$, where $0 \leq R < 1$.  The theorem holds for any $R < r < 1$.  
\end{proof}

\begin{mylem}
	\label{lem:ratio_a_indep_labeled_dep_good}
	Let $A_1$ and $A_2$ be any disjoint feature sets such that features in $A_1$ and $A_2$ are uncorrelated in both classes, and $\Sigma_0^{A}$ and $\Sigma_1^{A}$ are full rank, where $A = A_1 \cup A_2$.  
	Suppose for all $f \in F$ the fourth order moment across both classes exists and is finite. Let $S_{\infty}$ be a balanced sample.  
	Then, w.p.~$1$ there exist $c, N>0$ such that 
	\begin{align}
	\left( \frac
	{|C^{A_1}_0|^{\kappa^{A_1*}_0} 
	|C^{A_2}_0|^{\kappa^{A_2*}_0}
	|C^{A_1}_1|^{\kappa^{A_1*}_1}
	|C^{A_2}_1|^{\kappa^{A_2*}_1}}
	{|C^{A}_0|^{\kappa^{A*}_0}
	|C^{A}_1|^{\kappa^{A*}_1}}
	\right) ^{0.5}
	< (\log n)^c
	\label{eq:ratio_split_good}
	\end{align}
	for all $n > N$.  
	The left hand side of~\eqref{eq:ratio_split_good} is $a(P)/a(P')$ when $P$ and $P'$ are identical feature partitions except $A$ is a good block in $P$, and $A_1$ and $A_2$ are good blocks in $P'$.  
\end{mylem}

\begin{proof}
	We have that 
	\begin{align}
	&\left( \frac
	{|C^{A_1}_0|^{\kappa^{A_1*}_0} 
		|C^{A_2}_0|^{\kappa^{A_2*}_0}
		|C^{A_1}_1|^{\kappa^{A_1*}_1}
		|C^{A_2}_1|^{\kappa^{A_2*}_1}}
	{|C^{A}_0|^{\kappa^{A*}_0}
		|C^{A}_1|^{\kappa^{A*}_1}}
	\right) ^{0.5} \notag \\
	&\qquad= \left( \frac
	{|C^{A_1}_0|^{\kappa^{A_1}_0} 
		|C^{A_2}_0|^{\kappa^{A_2}_0}}
	{|C^{A}_0|^{\kappa^{A}_0}}
	\times 
	\frac
	{|C^{A_1}_1|^{\kappa^{A_1}_1}
		|C^{A_2}_1|^{\kappa^{A_2}_1}}
	{|C^{A}_1|^{\kappa^{A}_1}}
	\right) ^{0.5}
	\times R_{0, n} R_{1, n},
	\label{eq:a_ratio_split_good}
	\end{align}
	where 
	\begin{align}
	R_{y,n} = 
	\left(\frac
	{|C^{A_1}_y||C^{A_2}_y|}
	{|C^{A}_y|}
	\right)^{0.5n_y}.
	\label{eq:R_yn}
	\end{align}
	By Corollary~\ref{lem:convergence_of_C}, the first term in~\eqref{eq:a_ratio_split_good} converges to a positive finite constant w.p.~$1$. 
	Fix $y \in \{0, 1\}$.  Without loss of generality, assume features in $A$ are ordered such that 
	\begin{equation}
	C^A_y =
	\begin{pmatrix}
	C^{A_1}_y & V_y \\
	V_y^T & C^{A_2}_y
	\end{pmatrix}
	\end{equation}
	for some matrix $V_y$. 
	Observe that 
	\begin{align}
	|C^A_y|=|C^{A_1}_y| |C^{A_2}_y - W_y|,
	\label{eq:decompose_CAy}
	\end{align}
	where $W_y \equiv V^T_y (C^{A_1}_y)^{-1} V_y$.  
	Since features in $A_1$ and $A_2$ are uncorrelated in both classes, 
	\begin{equation}
	\Sigma^A_y =
	\begin{pmatrix}
	\Sigma^{A_1}_y & 0_{|A_1|, |A_2|} \\
	0_{|A_2|, |A_1|} & \Sigma^{A_2}_y
	\end{pmatrix}.
	\end{equation}
	By Lemma~\ref{lem:convergence_rate_of_C}, w.p.~$1$ there exist $K_1, N_1 > 0$ such that for all $n>N_1$, $y=0,1$, $i=1,\ldots,|A_1|$ and $j=1,\ldots,|A_2|$ we have 
	\begin{align}
	\label{eq:Vrate}
	|V_y(i,j)|<K_1 \sqrt{\frac{\log \log n}{n}}.
	\end{align}
	Since $W_y$ is comprised of only quadratic terms in $V_y$ and $C^{A_1}_y(i,j) \to \Sigma^{A_1}_y(i,j)$ w.p.~$1$ (Corollary~\ref{lem:convergence_of_C}), w.p.~$1$ there exist $K_2, N_2 > 0$ such that for all $n>N_2$, $y=0,1$, $i=1,\ldots,|A_1|$ and $j=1,\ldots,|A_2|$
	\begin{align}
	|W_y(i,j)| < K_2 \frac{\log \log n}{n}.
	\end{align}
	Further, since $|C^{A_2}_y| -  |C^{A_2}_y - W_y|$ is a polynomial function of the elements of $W_y$, where each term has a degree between $1$ and $|A_2|$ and a coefficient that is a polynomial function of the elements of $C^{A_2}_y$ of at most degree $|A_2|-1$, and since $C^{A_2}_y(i,j) \to \Sigma^{A_2}_y(i,j)$ w.p.~$1$ (Corollary~\ref{lem:convergence_of_C}), w.p.~$1$ there exists $K_3,N_3>0$ such that for all $n>N_3$ and $y = 0, 1$, 
	\begin{align}
	\left|  |C^{A_2}_y| -  |C^{A_2}_y - W_y| \right| < K_3 \frac{\log \log n}{n}.  
	\label{eq:CA2y_bound}
	\end{align}
	Therefore, 
	\begin{align}
	R_{y,n}
	&=
	\left(\frac{|C^{A_1}_y||C^{A_2}_y|}{|C^{A_1}_y| |C^{A_2}_y - W_y|}\right)^{0.5 n_y} \nonumber \\ 
	&<
	\left(1-\frac{K_3}{|C^{A_2}_y|}\frac{\log \log n}{n}\right)^{-0.5 n_y}
	\end{align}
	for $n>N_3$ w.p.~$1$, where the first line follows from~\eqref{eq:R_yn} and~\eqref{eq:decompose_CAy} and the second line from~\eqref{eq:CA2y_bound}.  Further, w.p.~$1$ there exists $N_4>N_3$ such that for all $n>N_4$ and $y = 0, 1$,
	\begin{align}
	\label{eq:Rnyfinal}
	R_{y,n}
	&<
	\left( 1 - \frac{2K_3}{|\Sigma^{A_2}_y|} \frac{\log \log n}{n}  \right)^{-0.5 n_y} \nonumber \\ 
	&<
	\left( 1 + \frac{4K_3}{|\Sigma^{A_2}_y|} \frac{\log \log n}{n}  \right)^{0.5 n_y} \nonumber \\ 
	&<
	\left( 1 + \frac{4 p_y K_3}{|\Sigma^{A_2}_y|} \frac{\log \log n}{n_y}  \right)^{0.5 n_y} \nonumber \\ 
	&<
	(\log n)^{2 p_y K_3/|\Sigma^{A_2}_y|}.  
	\end{align}
	The first line holds as long as $N_4$ is large enough that $|C^{A_2}_y|>|\Sigma^{A_2}_y|/2$ for all $n>N_4$ and $y=0, 1$ (this is possible since both $C^{A_2}_y$ converge).  
	The second line holds as long as $N_4$ is large enough that $x = (2 K_3/ |\Sigma^{A_2}_y|)\log\log n / n$ is between $0$ and $0.5$ for all $n>N_4$ and $y=0,1$, so that $(1-x)^{-1} < 1+2x$.  
	The third line holds as long as $N_4$ is large enough that $n_y/n < p_y$ for all $n>N_4$ and $y=0, 1$ and some $0 < p_0,p_1 < 1$ (this is possible since $S_\infty$ is balanced).  
	Finally, the last line follows from the fact that $(1+t/x)^x < e^t$ for all $x,t>0$.  
	From~\eqref{eq:a_ratio_split_good}, the theorem holds with $c=K_4 + 2 p_0 K_3/|\Sigma^{A_2}_0| + 2 p_1 K_3/|\Sigma^{A_2}_1|$ and $N=N_4$, where $K_4$ is such that $(\log n)^{K_4}$ exceeds a bound on the first term in~\eqref{eq:a_ratio_split_good}.  
\end{proof}

\begin{mylem}
	\label{lem:ratio_a_bad_labeled_good}
	Let $B$ be any feature set such that $\mu^B \equiv \mu_0^B = \mu_1^B$, $\Sigma^B \equiv \Sigma_0^B = \Sigma_1^B$, and $\Sigma^{B}$ is full rank.  
	Suppose for all $f \in F$ the fourth order moment across both classes exists and is finite. 
	Let $S_{\infty}$ be a balanced sample.  Then, w.p.~$1$ there exist $c, N>0$ such that 
	\begin{align}
	\left(
	\frac
	{|C^{B}|^{\kappa^{B*}}}
	{|C^{B}_0|^{\kappa^{B*}_0} |C^{B}_1|^{\kappa^{B*}_1}}
	\right) ^{0.5}
	< (\log n)^c
	\label{eq:ratio_correct_bad}
	\end{align}
	for all $n > N$.  The left hand side of~\eqref{eq:ratio_correct_bad} is $a(P)/a(P')$ when $P$ and $P'$ are identical feature partitions except $B$ is a good block in $P$ and a bad block in $P'$.  
\end{mylem}

\begin{proof}
	We have that 
	\begin{align}
	\left(
	\frac
	{|C^{B}|^{\kappa^{B*}}}
	{|C^{B}_0|^{\kappa^{B*}_0} |C^{B}_1|^{\kappa^{B*}_1}}
	\right) ^{0.5}
	&= \left(
	\frac
	{|C^{B}|^{\kappa^{B}}}
	{|C^{B}_0|^{\kappa^{B}_0} |C^{B}_1|^{\kappa^{B}_1}}
	\right) ^{0.5}
	\times R_n,
	\label{eq:a_ratio_correct_bad}
	\end{align}
	where 
	\begin{align}
	R_n = 
	\left( 
	\frac
	{|C^{B}|}
	{|C^{B}_0|^{\rho_n} |C^{B}_1|^{1-\rho_n}}
	\right)^{0.5n}.  
	\label{eq:Rn_thm_8}
	\end{align}
	By Corollary~\ref{lem:convergence_of_C}, the first term in~\eqref{eq:a_ratio_correct_bad} converges to a positive finite constant w.p.~$1$.  
	By Lemma~\ref{lem:convergence_rate_of_C}, w.p.~$1$ there exists $K_1, N_1>0$ such that for all $n> N_1$ and $y = 0, 1$,
	\begin{align}
	\left|\frac{|C^B_y|}{|\hat{\Sigma}^B_y|} - 1 \right| &<\frac{K_1}{n}, 
	\label{eq:Cy_bound_thm8} \\
	\left|\frac{|C^B|}{|\hat{\Sigma}^B|} - 1 \right| &<\frac{K_1}{n}.
	\label{eq:C_bound_thm8}
	\end{align}
	Let $N_2 > N_1$ be such that $x = K_1 / n$ is between $0$ and $0.5$ for all $n > N_2$, so that $(1-x)^{-1} < 1+2x$.  Then for all $n>N_2$, 
	\begin{align}
	R_n &< 
	\left( \frac
	{|\hat{\Sigma}^{B}|\left(1 + \frac{K_1}{n}\right)}
	{|\hat{\Sigma}_0^{B}|^{\rho_n}\left(1 - \frac{K_1}{n}\right)^{\rho_n} |\hat{\Sigma}_1^{B}|^{1-\rho_n}\left(1 - \frac{K_1}{n}\right)^{1-\rho_n}}
	\right)^{0.5n} \notag\\
	&< 	
	\left(1 + \frac{K_1}{n}\right)^{0.5n}
	\left(1 + \frac{2K_1}{n}\right)^{0.5 n \rho_n}
	\left(1 + \frac{2K_1}{n}\right)^{0.5 n (1-\rho_n)}
	\left( \frac
	{|\hat{\Sigma}^{B}|}
	{|\hat{\Sigma}_0^{B}|^{\rho_n} |\hat{\Sigma}_1^{B}|^{1-\rho_n}}
	\right)^{0.5n} \notag\\
	&< e^{1.5K_1} \left( 	\frac
	{|\hat{\Sigma}^B|}
	{|\hat{\Sigma}^B_0|^{\rho_n} |\hat{\Sigma}^B_1|^{1-\rho_n}}
	\right)^{0.5n}
	\label{eq:Rn6_e1}
	\end{align}	
	w.p.~$1$, where the last line follows from the fact that $(1+t/x)^x < e^t$ for all $x,t>0$.  
	From~\eqref{eq:Sigma_mixed_class}, 
	\begin{align}
	\hat{\Sigma}^B
	&= \Sigma^B + \rho_n E_0 + (1-\rho_n) E_1 + E,
	\end{align}
	and for $y = 0,1$, 
	\begin{align}
	E_y &= \hat{\Sigma}^B_y - \Sigma^B, \\
	E &= \frac{1-\rho_n}{n-1} \hat{\Sigma}^B_0 + \frac{\rho_n}{n-1} \hat{\Sigma}^B_1
	+ \frac{\rho_n (1-\rho_n) n}{n-1}(e_0 - e_1)(e_0 - e_1)^T,
	\end{align}
	where $e_y = \hat{\mu}_y^B - \mu^B$.
	By Lemma~\ref{lem:convergence_rate_of_C}, w.p.~$1$ there exists $K_2>0$ and $N_3>N_2$ such that 
	\begin{equation}
	|E_y(i,j)| < K_2 \sqrt{\frac{\log \log n}{n}}
	\label{eq:Ey_bound}
	\end{equation}
	for all $n>N_3$, $y=0,1$ and $i,j = 1, \ldots, |B|$.  
	Since $S_\infty$ is balanced, there exist $0 < p_0,p_1 < 1$ and $N_4 > N_3$ such that $n_y/n < p_y$ for all $n>N_4$ and $y=0, 1$.  Since $\hat{\Sigma}_y^B \to \Sigma^B$ as $n \to \infty$ w.p.~$1$, there exists $B > 0$ such that $|\hat{\Sigma}_y^B(i,j)|<B$ for all $n$, $y$, $i$ and $j$ w.p.~$1$.  
	By Lemma~\ref{lem:convergence_rate_of_C}, w.p.~$1$ there exists $K_3>0$ and $N_5>N_4$ such that $|e_y(i)| < K_3 \sqrt{\log \log n/n}$ for all $n> N_5$ and $y = 0, 1$.  
	By the triangle inequality, 
	\begin{align}
	|E(i,j)| 
	&\leq \frac{1-\rho_n}{n-1} |\hat{\Sigma}^B_0(i,j)| + \frac{\rho_n}{n-1} |\hat{\Sigma}^B_1(i,j)| \notag \\
	&\quad + \frac{\rho_n (1-\rho_n) n}{n-1}
	\left(|e_0(i)| + |e_1(i)|\right) 
	\left(|e_0(j)| + |e_1(j)|\right)  \notag \\
	&\leq \frac{p_1 B}{n-1} + \frac{p_0 B}{n-1} 
	+ 4 p_0 p_1 K_3^2 \frac{\log\log n}{n-1}.  
	\end{align}
	for all $n>N_5$ and $i,j = 1, \ldots, |B|$ w.p.~$1$.  In particular, there exists $K_4>0$ such that 
	\begin{align}
	|E(i,j)| 
	&< K_4 \frac{\log\log n}{n}.  
	\label{eq:E_bound}
	\end{align}
	Observe that $|\hat{\Sigma}^B|$ is a polynomial function of the elements of  $\Sigma^B$, $\rho_n E_0$, $(1-\rho_n)E_1$ and $E$, where each term has a degree of $|B|$ and a coefficient of $\pm 1$.  In particular, 
	\begin{align}
	|\hat{\Sigma}^B| 
	&= \sum_{i_1 = 1}^{|B|} \cdots \sum_{i_{|B|} = 1}^{|B|} \varepsilon_{i_1, \ldots, i_{|B|}} \hat{\Sigma}^B(1, i_1) \cdots \hat{\Sigma}^B(|B|, i_{|B|}) \notag \\
	&= \sum_{i_1 = 1}^{|B|} \cdots \sum_{i_{|B|} = 1}^{|B|} \varepsilon_{i_1, \ldots, i_{|B|}} 
	\sum_{k_1 = 1}^{4} \cdots \sum_{k_{|B|} = 1}^{4}
	X_{k_1}(1, i_1) \cdots X_{k_{|B|}}(|B|, i_{|B|}) \notag \\
	&= \sum_{k_1 = 1}^{4} \cdots \sum_{k_{|B|} = 1}^{4} m(k_1, \ldots, k_{|B|}),
	\end{align}
	where $\varepsilon_{i_1, \ldots, i_{|B|}}$ is the Levi-Civita symbol, equal to $+1$ if $(i_1, \ldots, i_{|B|})$ is an even permutation of $(1, \ldots, |B|)$, $-1$ if its an odd permutation, and $0$ otherwise, $X_1 = \Sigma^B$, $X_2 = \rho_n E_0$, $X_3 = (1-\rho_n) E_1$, $X_4 = E$, and 
	\begin{align}
	m(k_1, \ldots, k_{|B|}) = \sum_{i_1 = 1}^{|B|} \cdots \sum_{i_{|B|} = 1}^{|B|} \varepsilon_{i_1, \ldots, i_{|B|}} 
	X_{k_1}(1, i_1) \cdots X_{k_{|B|}}(|B|, i_{|B|}).  
	\end{align}
	From~\eqref{eq:Ey_bound} and~\eqref{eq:E_bound}, w.p.~$1$ there exists $K_5 \in \mathbb{R}$ such that 
	\begin{align}
	|\hat{\Sigma}^B| 
	&= m(1, \ldots, 1) 
	+ \sum_{(k_1, \ldots, k_{|B|}) \in \mathcal{M}_2} m(k_1, \ldots, k_{|B|})
	+ \sum_{(k_1, \ldots, k_{|B|}) \in \mathcal{M}_3} m(k_1, \ldots, k_{|B|}) \notag \\
	&\quad + \sum_{(k_1, \ldots, k_{|B|}) \in \mathcal{M}} m(k_1, \ldots, k_{|B|}) \notag \\
	&< m(1, \ldots, 1) 
	+ \sum_{(k_1, \ldots, k_{|B|}) \in \mathcal{M}_2} m(k_1, \ldots, k_{|B|})
	+ \sum_{(k_1, \ldots, k_{|B|}) \in \mathcal{M}_3} m(k_1, \ldots, k_{|B|}) \notag \\
	&\quad + K_5 \frac{\log \log n}{n}
	\label{eq:det_hat_Sigma}
	\end{align}
	for all $n > N_5$, where 
	\begin{align}
	\mathcal{M}_2 &= \{(k_1, \ldots, k_{|B|}) : \text{exactly one $k$ equals $2$, the rest equal $1$}\}, \notag \\
	\mathcal{M}_3 &= \{(k_1, \ldots, k_{|B|}) : \text{exactly one $k$ equals $3$, the rest equal $1$}\}, \notag \\
	\mathcal{M} &= \{(k_1, \ldots, k_{|B|}) : \text{at least two $k$'s are in $\{2,3\}$, or at least one $k$ equals $4$}\}. \notag
	\end{align}
	Further, w.p.~$1$ there exists $K_6 \in \mathbb{R}$ such that
	\begin{align}
	\rho_n|\hat{\Sigma}_0^B|
	&= \rho_n|\Sigma^B + E_0| \notag \\
	&= \rho_n \sum_{k_1 = 1}^{2} \cdots \sum_{k_{|B|} = 1}^{2} m^{\prime}(k_1, \ldots, k_{|B|}) \notag \\
	&= \rho_n m(1, \ldots, 1) \notag \\
	&\quad + \sum_{(k_1, \ldots, k_{|B|}) \in \mathcal{M}_2} m(k_1, \ldots, k_{|B|}) 
	+ \rho_n \sum_{(k_1, \ldots, k_{|B|}) \in \mathcal{M}^{\prime}} m^{\prime}(k_1, \ldots, k_{|B|}) \notag \\
	&> \rho_n m(1, \ldots, 1) 
	+ \sum_{(k_1, \ldots, k_{|B|}) \in \mathcal{M}_2} m(k_1, \ldots, k_{|B|}) 
	- K_6 \frac{\log \log n}{n}
	\label{eq:det_hat_Sigma_0}
	\end{align}
	for all $n > N_5$, where 
	\begin{align}
	m^{\prime}(k_1, \ldots, k_{|B|})
	&= \sum_{i_1 = 1}^{|B|} \cdots \sum_{i_{|B|} = 1}^{|B|} \varepsilon_{i_1, \ldots, i_{|B|}} 
	X^{\prime}_{k_1}(1, i_1) \cdots X^{\prime}_{k_{|B|}}(|B|, i_{|B|}), \notag \\
	\mathcal{M}^{\prime} &= \{(k_1, \ldots, k_{|B|}) : \text{at least two $k$'s equal $2$, the rest equal $1$}\},
	\end{align}
	$X^{\prime}_1 = \Sigma^B$, $X^{\prime}_2 = E_0$, and we have used the facts that $m^{\prime}(1, \ldots, 1) = m(1, \ldots, 1)$ and $\rho_n m^{\prime}(k_1, \ldots, k_{|B|}) = m(k_1, \ldots, k_{|B|})$ when $(k_1, \ldots, k_{|B|}) \in \mathcal{M}_2$.  
	Similarly, w.p.~$1$ there exists $K_7 \in \mathbb{R}$ such that
	\begin{align}
	(1-\rho_n)|\hat{\Sigma}_1^B|
	&> (1-\rho_n) m(1, \ldots, 1) 
	+ \sum_{(k_1, \ldots, k_{|B|}) \in \mathcal{M}_3} m(k_1, \ldots, k_{|B|}) 
	- K_7 \frac{\log \log n}{n}
	\label{eq:det_hat_Sigma_1}
	\end{align}
	for all $n > N_5$.  Combining~\eqref{eq:det_hat_Sigma}, \eqref{eq:det_hat_Sigma_0} and \eqref{eq:det_hat_Sigma_1}, w.p.~$1$ there exists $K_8 \in \mathbb{R}$ such that
	\begin{align}
	|\hat{\Sigma}^B| 
	< \rho_n|\hat{\Sigma}_0^B| + (1-\rho_n)|\hat{\Sigma}_1^B| 
	+ K_8 \frac{\log \log n}{n}
	\end{align}
	for all $n > N_5$.  
	Thus, w.p.~$1$ there exists $K_{9} \in \mathbb{R}$ such that 
	\begin{align}
	\frac{|\hat{\Sigma}^B|}
	{|\hat{\Sigma}^B_0|^{\rho_n} |\hat{\Sigma}^B_1|^{1-\rho_n}}
	&< \frac{\rho_n|\hat{\Sigma}_0^B| + (1-\rho_n)|\hat{\Sigma}_1^B|}
	{|\hat{\Sigma}^B_0|^{\rho_n} |\hat{\Sigma}^B_1|^{1-\rho_n}}
	+ K_{9} \frac{\log \log n}{n} \notag \\
	&= \rho_n \left(\frac{|\hat{\Sigma}^B_0|}{|\hat{\Sigma}^B_1|}\right)^{1-\rho_n}
	+ (1-\rho_n) \left(\frac{|\hat{\Sigma}^B_1|}{|\hat{\Sigma}^B_0|}\right)^{\rho_n}
	+ K_{9} \frac{\log \log n}{n}
	\end{align}
	for all $n>N_5$, where $K_{9}$ is chosen based on the fact that $|\hat{\Sigma}^B_y| \to |\Sigma^B|$ w.p.~$1$ (an application of Corollary~\ref{lem:convergence_of_C} with hyperparameters $\nu_y^B=0$, $\kappa_y^B=|A|$ and $S_y^B = 0$ in place of the hyperparameters used by the selection rule), and thus $|\hat{\Sigma}^B_y|$ must be bounded for all $n$.  
	In addition, w.p.~$1$ there exists $K_{10} > 0$ and $N_6>N_5$ such that 
	\begin{align}
	\left|\frac{|\hat{\Sigma}^B_0|}{|\hat{\Sigma}^B_1|} - 1\right|
	&= \frac{||\hat{\Sigma}^B_0| - |\hat{\Sigma}^B_1||}{|\hat{\Sigma}^B_1|} \notag \\
	&\leq \frac{||\hat{\Sigma}^B_0| - |\Sigma^B|| + ||\hat{\Sigma}^B_1| + |\Sigma^B||}{0.5|\Sigma^B|} \notag \\
	&= 2 \left| \frac{|\hat{\Sigma}^B_0|}{|\Sigma^B|} - 1 \right|
	+ 2 \left| \frac{|\hat{\Sigma}^B_1|}{|\Sigma^B|} - 1 \right| \notag \\
	&< 4K_{10} \sqrt{\frac{\log\log n}{n}}
	\end{align}
	for $n > N_6$, 
	where in the second line we have applied the triangle inequality and the fact that $|\hat{\Sigma}_1^B| \to |\Sigma^B|$ w.p.~$1$, so $N_6$ is chosen such that $|\hat{\Sigma}^B_1| > 0.5|\Sigma^B|$ for all $n > N_6$, and the last line follows from Lemma~\ref{lem:convergence_rate_of_C}.
	By Lemma~S3 in~\cite{igib1}, there exists $r \in (0,1)$ such that for all $\rho \in (0, 1)$ and $x \in (1-r, 1+r)$, 
	\begin{equation}
	\rho x^{1-\rho} + (1-\rho) x^{-\rho} 
	\leq 1 + \rho(1-\rho)(x-1)^2
	\leq 1 + 0.25(x-1)^2,  
	\end{equation}
	where in the last inequality we use the fact that $\rho(1-\rho) < 0.25$ for all $\rho \in [0,1]$.  
	Thus, w.p.~$1$ there exists $N_7>N_6$ such that 
	\begin{align}
	\frac{|\hat{\Sigma}^B|}
	{|\hat{\Sigma}^B_0|^{\rho_n} |\hat{\Sigma}^B_1|^{1-\rho_n}}
	&< 1 + 0.25 \left(\frac{|\hat{\Sigma}^B_0|}{|\hat{\Sigma}^B_1|} - 1\right)^{2}
	+ K_{9} \frac{\log \log n}{n} \notag \\
	&< 1 + K_{11} \frac{\log \log n}{n}
	\end{align}
	for all $n > N_7$, where $K_{11} = 4 K_{10}^2 + K_{9}$.  
	Thus, from~\eqref{eq:Rn6_e1}, w.p.~$1$
	\begin{align}
	R_n 
	&< e^{1.5K_1} \left( 1 + K_{11} \frac{\log \log n}{n} \right)^{0.5n} \notag \\
	&< e^{1.5K_1} (\log n)^{0.5K_{11}}
	\end{align}
	for all $n>N_7$, where the last line follows from the fact that $(1+t/x)^x < e^t$ for all $x,t>0$.  From~\eqref{eq:a_ratio_correct_bad}, the theorem holds with $c = K_{12} + 0.5K_{11}$ and $N = N_7$, where $K_{12}$ is such that $(\log n)^{K_{12}}$ exceeds $e^{1.5K_1}$ times a bound on the first term in~\eqref{eq:a_ratio_correct_bad}.  
\end{proof}

\begin{mylem}
	\label{lem:ratio_a_dep_labeled_indep_good}
	Let $G_1$ and $G_2$ be any disjoint feature sets such that there exists at least one feature in $G_1$ and one feature in $G_2$ that are correlated in at least one class, and $\Sigma_0^{G}$ and $\Sigma_1^{G}$ are full rank, where $G = G_1 \cup G_2$.  
	Let $S_{\infty}$ be a balanced sample.  Then, w.p.~$1$ there exist $0 \leq r < 1$ and $N>0$ such that
	\begin{align}
	\left(
	\frac{|C^{G}_0|^{\kappa^{G*}_0}
	|C^{G}_1|^{\kappa^{G*}_1}}
	{|C^{G_1}_0|^{\kappa^{G_1*}_0}
	|C^{G_2}_0|^{\kappa^{G_2*}_0}
	|C^{G_1}_1|^{\kappa^{G_1*}_1}
	|C^{G_2}_1|^{\kappa^{G_2*}_1}}
	\right)^{0.5}
	< r^n
	\label{eq:ratio_merge_good}
	\end{align}
	for all $n > N$.  The left hand side of~\eqref{eq:ratio_merge_good} is $a(P)/a(P')$ when $P$ and $P'$ are strict refinements of the unambiguous feature partition that are identical except $G_1$ and $G_2$ are good blocks in $P$ and $G$ is a good block in $P'$.  
\end{mylem}

\begin{proof}
	We have that 
	\begin{align}
	&\left(
	\frac{|C^{G}_0|^{\kappa^{G*}_0}
		|C^{G}_1|^{\kappa^{G*}_1}}
	{|C^{G_1}_0|^{\kappa^{G_1*}_0}
		|C^{G_2}_0|^{\kappa^{G_2*}_0}
		|C^{G_1}_1|^{\kappa^{G_1*}_1}
		|C^{G_2}_1|^{\kappa^{G_2*}_1}}
	\right)^{0.5} \notag \\
	&\qquad= \left(
	\frac{|C^{G}_0|^{\kappa^{G}_0}}
	{|C^{G_1}_0|^{\kappa^{G_1}_0}|C^{G_2}_0|^{\kappa^{G_2}_0}}
	\times
	\frac{|C^{G}_1|^{\kappa^{G}_1}}
	{|C^{G_1}_1|^{\kappa^{G_1}_1}|C^{G_2}_1|^{\kappa^{G_2}_1}}
	\right)^{0.5}
	\times  
	R_{0, n}^{n_0} R_{1, n}^{n_1},
	\label{eq:a_ratio_merge_good}
	\end{align}
	where
	\begin{align}
	R_{y, n} 
	= \left(\frac{|C^{G}_y|}
	{|C^{G_1}_y||C^{G_2}_y|}\right)^{0.5}.
	\end{align}
	By Corollary~\ref{lem:convergence_of_C}, the first term in~\eqref{eq:a_ratio_merge_good} converges to a positive finite constant w.p.~$1$, and~\eqref{eq:R_y_merge_good_diff} holds.  The rest of the proof proceeds exactly as in Lemma~\ref{lem:ratio_a_good_labeled_bad_merge_same}.  
\end{proof}

\begin{mylem}
	\label{lem:ratio_a_dep_labeled_indep_bad}
	Let $B_1$ and $B_2$ be any disjoint feature sets such that $\mu_0^B = \mu_1^B$, $\Sigma^B \equiv \Sigma_0^B = \Sigma_1^B$, there exists at least one feature in $B_1$ and one feature in $B_2$ that are correlated in at least one class, and $\Sigma^{B}$ is full rank, 
	where $B = B_1 \cup B_2$.  
	Let $S_{\infty}$ be a balanced sample.  Then, w.p.~$1$ there exist $0 \leq r < 1$ and $N>0$ such that 
	\begin{align}
	\left(
	\frac{|C^{B}|^{\kappa^{B*}}}
	{|C^{B_1}|^{\kappa^{B_1*}}|C^{B_2}|^{\kappa^{B_2*}}}
	\right)^{0.5}
	< r^n
	\label{eq:ratio_merge_bad}
	\end{align}
	for all $n > N$.  The left hand side of~\eqref{eq:ratio_merge_bad} is $a(P)/a(P')$ when $P$ and $P'$ are strict refinements of the unambiguous feature partition that are identical except $B_1$ and $B_2$ are bad blocks in $P$ and $B$ is a bad block in $P'$.  
\end{mylem}

\begin{proof}
	We have that
	\begin{align}
	\left(
	\frac{|C^{B}|^{\kappa^{B*}}}
	{|C^{B_1}|^{\kappa^{B_1*}}|C^{B_2}|^{\kappa^{B_2*}}}
	\right)^{0.5}
	&= \left(
	\frac{|C^{B}|^{\kappa^{B}}}
	{|C^{B_1}|^{\kappa^{B_1}}|C^{B_2}|^{\kappa^{B_2}}}
	\right)^{0.5}
	\times  
	R_n^{n}, 
	\label{eq:a_ratio_merge_bad}
	\end{align}
	where 
	\begin{equation}
	R_n = \left(\frac{|C^{B}|}{|C^{B_1}||C^{B_2}|}\right)^{0.5}.
	\end{equation}	
	By Corollary~\ref{lem:convergence_of_C}, the first term in~\eqref{eq:a_ratio_merge_bad} converges to a positive finite constant w.p.~$1$, and 
	\begin{align}
	R_n \to
	\left( 
	\frac{|\Sigma^{B}|}
	{|\Sigma^{B_1}||\Sigma^{B_2}|}
	\right)^{0.5}
	\equiv R
	\end{align}
	as $n \to \infty$, both w.p.~$1$, where $\Sigma^{B_1} \equiv \Sigma^{B_1}_0 = \Sigma^{B_1}_1$ and $\Sigma^{B_2} \equiv \Sigma^{B_2}_0 = \Sigma^{B_2}_1$.
	Without loss of generality, assume features in $B$ are ordered such that 
	\begin{equation}
	\Sigma^{B} =
	\begin{pmatrix}
	\Sigma^{B_1} & V \\
	V^T & \Sigma^{B_2}
	\end{pmatrix}, 
	\end{equation}
	where $V$ is a matrix of correlations between features in $B_1$ and $B_2$.  
	Since 
	\begin{equation}
	|\Sigma^{B}| = |\Sigma^{B_1}| |\Sigma^{B_2} - V^T(\Sigma^{B_1})^{-1}V|,
	\end{equation}
	we have that 
	\begin{align}
	R =
	\left( 
	\frac{|\Sigma^{B_2} - V^T(\Sigma^{B_1})^{-1}V|}
	{|\Sigma^{B_2}|}
	\right)^{0.5}.
	\end{align}
	Since there exists at least one feature in $B_1$ and one feature in $B_2$ that are correlated, $V$ is non-zero, $|\Sigma^{B_2} - V^T(\Sigma^{B_1})^{-1}V| < |\Sigma^{B_2}|$, and $0 \leq R < 1$.  
	The theorem holds for any $R < r < 1$.  
\end{proof}

\section{Conclusion}
\label{sec:conclusion}

The consistency theory presented herein is important because it provides a richer understanding the type of features selected by OBFS.  
Furthermore, we have characterized rates of convergence for the posterior on feature partitions, and the marginal posterior probabilities on individual features.  

Although here we focus on identifying $\bar{G}$ using the posterior $\pi^*(G)$, the OBFS framework can be used for \emph{optimal Bayesian partition selection}, which aims to identify $\bar{P}$ using the full posterior $\pi^*(P)$.  Partition selection may be of interest, for instance, if one wishes to identify communities of closely interacting genes.  Since Theorem~\ref{sec:thm_conv_1} proves that $\pi^*(P)$ converges to a point mass at $\bar{P}$, this theorem has direct implications on the consistency of optimal Bayesian partition selection as well.  

The conditions in Theorem~\ref{sec:thm_conv_1} are sufficient but not necessary.  
For example, it may be possible to relax Condition (iii) of Theorem~\ref{sec:thm_conv_1}.  
It is also possible for $\pi^*(G)$ to converge to a point mass at $\bar{G}$, but for $\pi^*(P)$ to not converge to a point mass at $\bar{P}$.  
For example, if non-zero correlations are present in the data generation process, then the OBF variant of OBFS sets $\pi(\bar{P}) = 0$, which means that Condition (v) of Theorem~\ref{sec:thm_conv_1} does not hold.  
However, if the independent unambiguous set of good features given in Definition~\ref{def:indep_unambiguous} and the unambiguous set of good features given in Definition~\ref{sec:ded} happen to be equal (the latter always contains the former), then OBF can be strongly consistent relative to the unambiguous set of good features.  

OBFS searches for the unambiguous set of good features, which expands upon the independent unambiguous set of good features targeted by OBF.  In particular, the unambiguous set of good features includes features that are only strongly discriminating when grouped together, features that are individually weak but strongly correlated with strong features, and features that are linked to discriminating features only through a chain of correlation.  The unambiguous feature set is complete in the sense that any features that are not included must not have any first or second order distributional differences between the classes and must be uncorrelated with all features that are included, and minimal in the sense that it is the smallest feature set with this property.  The unambiguous feature set is thus of great interest and importance in bioinformatics and many other application areas, although we are not aware of any selection algorithms besides OBFS that aim to identify this set.  






\begin{acknowledgement}
	This work is supported by the National Science Foundation (CCF-1422631 and CCF-1453563).
\end{acknowledgement}



\end{document}